\documentclass[10pt,journal,compsoc]{IEEEtran}
\ifCLASSOPTIONcompsoc
  \usepackage[nocompress]{cite}
\else
  \usepackage{cite}
\fi

\usepackage{amsmath}
\usepackage{amsthm}
\usepackage{amssymb}
\usepackage{dsfont}
\usepackage{graphicx}
\usepackage{combinedgraphics}
\usepackage{multirow}
\usepackage{booktabs}
\usepackage{xcolor}
\usepackage{comment}
\usepackage{mathrsfs}
\usepackage{wrapfig}
\usepackage{algorithm}
\usepackage{algorithmicx}  
\usepackage{algpseudocode}
\usepackage[caption=false]{subfig}
\usepackage{bbm}

\newtheorem{definition}{Definition}
\newtheorem{thm}{Theorem}

\newtheorem{prop}{Proposition}
\ifCLASSINFOpdf
\else
\fi
\begin{document}
%
\title{Isolation Distributional Kernel: \\A New Tool for Point \& Group Anomaly Detection}
%
%
%
%

\author{Kai Ming Ting,
        Bi-Cun Xu,
        Takashi Washio,
        and~Zhi-Hua Zhou,~\IEEEmembership{Fellow,~IEEE}
\IEEEcompsocitemizethanks{\IEEEcompsocthanksitem K. M. Ting, B.-C. Xu and Z.-H. Zhou are with National Key Laboratory ofNovel Software Technology, Nanjing University, 210023, China.\protect\\
E-mail: \{tingkm,xubc,zhouzh\}@lamda.nju.edu.cn
\IEEEcompsocthanksitem T Washio is with the Institute of Scientific and Industrial Research, Osaka University, Japan

}
}

\IEEEtitleabstractindextext{%
\begin{abstract}
We introduce Isolation Distributional Kernel as a new way to measure the similarity between two distributions. Existing approaches based on kernel mean embedding, which convert a point kernel to a distributional kernel, have two key issues: the point kernel employed has a feature map with intractable dimensionality; and it is {\em data independent}.
This paper shows that Isolation Distributional Kernel (IDK), which is based on a {\em data dependent} point kernel, addresses both key issues.
We demonstrate IDK's efficacy and efficiency as a new tool for kernel based anomaly detection for both point and group anomalies. Without explicit learning, using IDK alone outperforms existing kernel based point anomaly detector OCSVM and other kernel mean embedding methods that rely on Gaussian kernel.
For group anomaly detection, 
we introduce an IDK based detector called IDK$^2$. 
It reformulates the problem of group anomaly detection in input space into the problem of point anomaly detection in Hilbert space, without the need for learning. 
IDK$^2$ runs orders of magnitude faster than group anomaly detector OCSMM.
We reveal for the first time that an effective kernel based anomaly detector based on kernel mean embedding must employ a characteristic kernel which is data dependent.

\end{abstract}

}

\maketitle

\IEEEdisplaynontitleabstractindextext

%
\IEEEpeerreviewmaketitle

\IEEEraisesectionheading{\section{Introduction}\label{sec:introduction}}

\IEEEPARstart{I}{n} many real-world applications, objects are naturally represented as groups of data points generated from one or more distributions \cite{EMK_Bo-NIPS2009,SupportMeasureMchines-NIPS2012,Sutherland-Thesis2016}. Examples are: (i) in the context of multi-instance learning, each object is represented as a bag of data points; and (ii) in astronomy, a cluster of galaxies is a sample from a distribution, where other clusters of galaxies may also belong.

Kernel mean embedding \cite{HilbertSpaceEmbedding2007, KernelMeanEmbedding2017} on distributions is one effective way to build a distributional kernel from a point kernel, enabling similarity between distributions to be measured. The current approach has focused on point kernels which have {\em a feature map with intractable dimensionality}. This feature map is a known key issue in kernel mean embedding in the literature \cite{KernelMeanEmbedding2017}; and it has led to $O(n^2)$ time complexity, where $n$ is the input data size. 

Here, we identify that being {\em data independent} point kernel  is another key issue which compromises the effectiveness of the similarity measurement that impacts on task-specific performance. 
We propose to employ a {\em data dependent} point kernel to address the above two key issues directly. As it is implemented using Isolation Kernel \cite{ting2018IsolationKernel,IsolationKernel-AAAI2019}, we called the proposed kernel mean embedding: Isolation Distributional Kernel or IDK.

Our contributions are:
\begin{enumerate}
\setlength\itemsep{0em}
    \item Proposing a new implementation of Isolation Kernel which  has  the  required  data  dependent  property for  anomaly  detection. We provide a geometrical interpretation of the new Isolation Kernel in Hilbert Space that explains its superior anomaly detection accuracy over existing Isolation Kernel.

    \item Formally proving that the new Isolation Kernel (i) has the following data dependent property: \emph{two points, as measured by Isolation Kernel derived in a sparse region, are more similar than the same two points, as measured by Isolation Kernel derived in a dense region}; and (ii) is a characteristic kernel. We show that both properties are necessary conditions to be an effective kernel-based anomaly detector based on kernel mean embedding.
    
    \item Introducing Isolation Distributional Kernel (IDK). It is distinguished from existing distributional kernels in two aspects. First, the use of the data dependent Isolation Kernel produces high accuracy in anomaly detection tasks. Second, the Isolation Kernel's exact and finite-dimensional feature map enables IDK to have $O(n)$ time complexity.
    
    \item Proposing IDK point anomaly detector and IDK$^2$ group anomaly detector. IDK and IDK$^2$ are more effective and more efficient than existing Gaussian kernel based anomaly detectors. They run orders of magnitude faster and can deal with large scale datasets that their Gaussian kernel counterparts could not. Remarkably, the superior detection accuracy is achieved without explicit learning.
    
    

     \item  Revealing for the first time that {\em the problem of group anomaly detection in input space} can be effectively reformulated as {\em the problem of point anomaly detection in Hilbert space}, without explicit learning.
     
\end{enumerate}

\section{Background}
\label{sec_background}
In this section, we briefly describe the kernel mean embedding \cite{KernelMeanEmbedding2017} and Isolation Kernel \cite{ting2018IsolationKernel} as the background of the proposed method. The key symbols and notations used are shown in Table \ref{tbl_symbols}.

\subsection{Kernel mean embedding and two key issues}
\label{sec_kernel_mean_embedding}

\begin{table}[t]
		\centering
		\caption{Key symbols and notations.}
		\label{tbl_symbols}
		\begin{tabular}{ll}
			\toprule
			
			$x$ & a point in input space $\mathbb{R}^d$\\
			$\mathsf{G}$ & a set of points $\{x_i\ |\ i=1,\dots,m\}$ in $\mathbb{R}^d$,  $x \sim \mathcal{P}_\mathsf{G}$\\
		$\mathcal{P}_\mathsf{G}$ & a distribution that generates a set $\mathsf{G}$ of points in $\mathbb{R}^d$\\	
		
		$\widehat{\mathcal{K}}_G$& Gaussian Distributional Kernel (GDK)\\
		$\widehat{\varphi}$ & Approximate/exact feature map of GDK\\
		$\widehat{\mathcal{K}}_{NG}$& Nystr\"{o}m accelerated GDK using its feature map $\widehat{\varphi}$\\
		$\widehat{\mathcal{K}}_I$ & Isolation Distributional Kernel (IDK)\\
		$\widehat{\Phi}$ & Exact feature map of IDK\\
		$\mathbf{g}$ & $\mathsf{G}$ is mapped to a point $\mathbf{g}=\widehat{\Phi}(\mathcal{P}_\mathsf{G})$ in Hilbert Space  $\mathscr{H}$\\
		$\Pi$ &  a set of points $\{\mathbf{g}_j\ |\ j=1,\dots,n\}$ in $\mathscr{H}$, $\mathbf{g} \sim \mathbf{P}_\Pi$\\
		$\mathbf{P}_\Pi$ & 	 a distribution that generates a set $\Pi$ of points in $\mathscr{H}$\\
		GDK$^2$ & Two levels of GDK with Nystr\"{o}m approximated  $\widehat{\varphi}$ \& $\widehat{\varphi}_2$\\
		IDK$^2$ & Two levels of IDK with $\widehat{\Phi}$ \& $\widehat{\Phi}_2$\\
			\bottomrule
		\end{tabular}
	\end{table}

Let $S$ and $T$ be two nonempty datasets where each point $x$ in $S$ and $T$ belongs to a subspace $\mathcal{X} \subseteq \mathbb{R}^d$ and is drawn from probability distributions $\mathcal{P}_S$ and $\mathcal{P}_T$ defined on $\mathbb{R}^d$, respectively. $\mathcal{P}_S$ and $\mathcal{P}_T$ are strictly positive on $\mathcal{X}$ and strictly zero on $\overline{\mathcal{X}}=\mathbb{R}^d \setminus \mathcal{X}$, {\it i.e.}, $\forall X \subseteq \mathcal{X} \mbox{ s.t. } X \neq \emptyset; \mathcal{P}_S(X), \mathcal{P}_T(X)>0$, and $\forall X \subseteq \overline{\mathcal{X}} \mbox{ s.t. } X \neq \emptyset; \mathcal{P}_S(X), \mathcal{P}_T(X)=0$. We denote the density of $\mathcal{P}_S$ and $\mathcal{P}_T$ as $\mathcal{P}_S(x)$ and $\mathcal{P}_T(x)$, respectively.

Using kernel mean embedding \cite{HilbertSpaceEmbedding2007,KernelMeanEmbedding2017},  the empirical estimation of the distributional kernel $\widehat{\mathcal{K}}$ on $\mathcal{P}_S$ and $\mathcal{P}_T$, which is based on a point kernel $\kappa$ on points $x,y \in \mathcal{X}$, is given as:
  \begin{equation}
	\widehat{\mathcal{K}}_G(\mathcal{P}_S,\mathcal{P}_T) =\frac{1}{|S||T|}\sum_{x \in S} \sum_{y \in T} \kappa(x,y). 
	\label{eqn_mmk}
  \end{equation}

\subsubsection*{First issue: Feature map has intractable dimensionality}
The distributional kernel $\widehat{\mathcal{K}}_G$ relies on a point kernel $\kappa$, e.g., Gaussian kernel, which has a feature map with intractable dimensionality. 
The use of a  point kernel which has {\em a feature map with intractable dimensionality is regarded as a fundamental issue of kernel mean embedding} \cite{KernelMeanEmbedding2017}. This can be seen from Eq~\ref{eqn_mmk} which has time complexity $O(n^2)$ if each set of $S$ and $T$ has data size $n$.

It has been recognised that the time complexity can be reduced by utilising the feature map of the point kernel \cite{KernelMeanEmbedding2017}. 

If the chosen point kernel can be approximated as $\kappa(x,y) \approx \left<\varphi(x), \varphi(y)\right>$, where $\varphi$ is a finite-dimensional feature map approximating the feature map of $\kappa$. Then, $\widehat{\mathcal{K}}$ can be written as 	
  \begin{equation}
\widehat{\mathcal{K}}_{NG}(\mathcal{P}_S,\mathcal{P}_T)  \approx  \frac{1}{|S||T|}\sum_{x\in S} \sum_{y\in T} \varphi(x)^\top\varphi(y)  
  =  \left< \widehat{\varphi}(\mathcal{P}_S), \widehat{\varphi}(\mathcal{P}_T) \right>
   \label{eqn_MMK_approximation}
  \end{equation}
where $\widehat{\varphi}(\mathcal{P}_T) = \frac{1}{|T|} \sum_{x \in T} \varphi(x)$ is the empirical estimation of the approximate feature map of $\widehat{\mathcal{K}}_G(\mathcal{P}_T, \cdot)$, or equivalently, the approximate kernel mean map of $\mathcal{P}_T$ in RKHS (Reproducing Kernel Hilbert Space)  $\mathscr{H}$ associated with $\widehat{\mathcal{K}}_G$.

Note that the approximation $\kappa(x,y) \approx \left<\varphi(x), \varphi(y)\right>$ is essential in order to have a finite-dimensional feature map. 
\textbf{This enables the use of Eq~\ref{eqn_MMK_approximation} to reduce the time complexity of computing $\widehat{\mathcal{K}}(\mathcal{P}_S,\mathcal{P}_T)$ to $O(n)$} since $\widehat{\varphi}(\mathcal{P})$ can be computed independently in $O(n)$. Otherwise, Eq \ref{eqn_mmk} must be used which costs $O(n^2)$.
A successful approach of finite-dimensional feature map approximation is kernel functional approximation. Representative methods are Nystr\"{o}m method \cite{Nystrom_NIPS2000} and Random Fourier Features \cite{RandomFeatures2007,Nystrom-NIPS12}. 

Existing distributional kernels such as the level-2 kernel\footnote{The level-2 kernel in OCSMM \cite{SupportMeasureMchines-NIPS2012} refers to the distributional kernel created from a point kernel as the level-1 kernel. This is different from our use of the term: both level-1 and level-2 IDKs are distributional kernels.} used in support measure machine (SMM) \cite{SupportMeasureMchines-NIPS2012}, Mean map kernel (MMK) \cite{Sutherland-Thesis2016} and Efficient Match Kernel (EMK) \cite{EMK_Bo-NIPS2009} have exactly the same form as shown in Eq \ref{eqn_mmk}, where $\kappa$ can be any of the existing data independent point kernels. Both MMK and EMK employ a kernel functional approximation in order to use Eq \ref{eqn_MMK_approximation}.

In summary, using a point kernel which has a feature map with intractable dimensionality,
the kernel functional approximation is an enabling step to approximate the point kernel with a finite-dimensional feature map. 
Otherwise, the mapping from $T$ in input space to a point $\widehat{\varphi}(\mathcal{P}_T)$ in Hilbert space cannot be performed; and Eq \ref{eqn_MMK_approximation} cannot be computed. However, these methods of kernel functional approximation are computationally expensive; and they almost always weaken the final outcome in comparison with that derived from Eq \ref{eqn_mmk}.

\subsubsection*{Second issue: kernel is data independent}

In addition to the known key issue mentioned above, we identify that {\em a data independent point kernel is a key issue} which 
leads to poor task-specific performance. 

The weakness of using a data independent kernel/distance is well recognised in the literature. For example, distance metric learning \cite{DistMetricLearning-Xing:2002,weinberger2009distance,zadeh2016geometric} aims to transform the input space such that points of the same class become closer and points of different classes are lengthened in the transformed space than those in the input space. Distance metric learning has been shown to improve the classification accuracy of k nearest neighbour classifiers \cite{DistMetricLearning-Xing:2002,weinberger2009distance,zadeh2016geometric}. 

A recent work has shown that data independent kernels such as Laplacian kernel and Gaussian Kernel are the source of weaker predictive SVM classifiers  \cite{ting2018IsolationKernel}. Unlike distance metric learning, it creates a data dependent kernel directly from data, requiring neither class information nor explicit learning. It also provides a reason why a data dependent kernel is able to improve the predictive accuracy of SVM that uses a data independent kernel.

Here we show that the use of data independent kernel reduces the effectiveness of kernel mean embedding in the context of anomaly detection. The resultant anomaly detectors which employ Gaussian kernel, using either $\widehat{\mathcal{K}}_G$ or $\widehat{\mathcal{K}}_{NG}$, perform poorly (see Section \ref{sec_experiments}.) This is because a data independent kernel is employed.


\subsection{Isolation Kernel}
\label{sec_IsolationKernel}
Let $D \subset \mathcal{X} \subseteq \mathbb{R}^d$ be a dataset sampled from an unknown $\mathcal{P}_D$; and $\mathbb{H}_\psi(D)$ denote the set of all partitionings $H$ that are admissible from $\mathcal{D} \subset D$, where each point $z \in \mathcal{D}$ has the equal probability of being selected from $D$; and $|\mathcal{D}|=\psi$.
Each partition $\theta[z] \in H$ isolates a point $z \in \mathcal{D}$ from the rest of the points in $\mathcal{D}$.
Let $\mathds{1}(\cdot)$ be an indicator function.

\begin{definition}\label{IKernel} \cite{ting2018IsolationKernel,IsolationKernel-AAAI2019} For any two points $x,y \in \mathbb{R}^d$,
	Isolation Kernel of $x$ and $y$ is defined to be
	the expectation taken over the probability distribution on all partitionings $H \in \mathds{H}_\psi(D)$ that both $x$ and $y$  fall into the same isolating partition $\theta[z] \in H$, where $z \in \mathcal{D} \subset D$, $\psi=|\mathcal{D}|$:
	\begin{eqnarray}
\kappa_I(x,y\ |\ D)  & = & {\mathbb E}_{\mathds{H}_\psi(D)} [\mathds{1}(x,y \in \theta[z]\ | \ \theta[z] \in H)] 
		\label{eqn_kernel}
	\end{eqnarray}
\end{definition}

In practice, $\kappa_I$ is constructed using a finite number of partitionings $H_i, i=1,\dots,t$, where each $H_i$ is created using randomly subsampled $\mathcal{D}_i \subset D$; and $\theta$ is a shorthand for $\theta[z]$:
\begin{eqnarray}
\kappa_I(x,y\ |\ D)  & = &  \frac{1}{t} \sum_{i=1}^t   \mathds{1}(x,y \in \theta\ | \ \theta \in H_i) \nonumber\\
 & = & 
 \frac{1}{t} \sum_{i=1}^t \sum_{\theta \in H_i}   \mathds{1}(x\in \theta)\mathds{1}(y\in \theta) 
 \label{Eqn_IK}
\end{eqnarray}

Isolation Kernel is positive semi-definite as Eq \ref{Eqn_IK} is a quadratic form. Thus, Isolation Kernel defines a  RKHS $\mathscr{H}$.

The isolation partitioning mechanisms which have been used previously to implement Isolation Kernel are iForest \cite{ting2018IsolationKernel}, and Voronoi diagram \cite{IsolationKernel-AAAI2019} (they are applied to SVM classifiers and density-based clustering.)

\section{Proposed Isolation Kernel}
\label{sec_proposed_IK}
Here we introduce a new implementation of Isolation Kernel, together with its exact and finite feature map, and its data dependent property in the following three subsections.

\subsection{A new implementation of Isolation Kernel}
\label{sec_IK_implementation}
An isolation mechanism for Isolation Kernel which has not been employed previously is given as follows: 

Each point $z \in \mathcal{D}$ is isolated from the rest of the points in $\mathcal{D}$ by building a hypersphere that covers $z$ only. The radius of the hypersphere
is determined by the distance between $z$ and its nearest neighbor in $\mathcal{D} \setminus \{z\}$. In other words, a partitioning $H$ consists of $\psi$ hyperspheres $\theta[z]$ and the $(\psi+1)$-th partition. The latter is the region in $\mathbb{R}^d$ which is not covered by all $\psi$ hyperspheres. Note that  $2 \le \psi < |D|$. Figure \ref{fig_example} (left) shows an example using $\psi=3$. 

This mechanism has been shown to produce large partitions in a sparse region and small partitions in a dense region \cite{iNNE}. Although it was used as a point anomaly detector called iNNE \cite{iNNE}, its use in creating a kernel is new.

\subsection{Feature map of the new Isolation Kernel}
\label{sec_properties}
Given a partitioning $H_i$, let $\Phi_i(x)$ be a $\psi$-dimensional binary column vector representing all hyperspheres $\theta_j \in H_i$, $j=1,\dots,\psi$; where $x$ falls into either only one of the $\psi$ hyperspheres or none.
The $j$-component of the vector is:
$\Phi_{ij}(x)=\mathds{1}(x\in \theta_j\ |\ \theta_j\in H_i)$. Given $t$ partitionings, $\Phi(x)$ is the concatenation of $\Phi_1(x),\dots,\Phi_t(x)$.

\begin{definition}
	\label{def:featureMap}
	\textbf{Feature map of Isolation Kernel.}
	For point $x \in \mathbb{R}^d$, the feature mapping $\Phi: x\rightarrow \mathbb \{0,1\}^{t\times \psi}$ of $\kappa_I$ is a vector that represents the partitions in all the partitioning $H_i\in \mathds{H}_\psi(D)$, $i=1,\dots,t$; where $x$ falls into either only one of the $\psi$ hyperspheres or none in each partitioning $H_i$.
\end{definition}

Let $\mathbbmtt{1}$ be a shorthand of $\Phi_i(x)$ such that $\Phi_{ij}(x)=1$ and $\Phi_{ik}(x)=0, \forall k \neq j$ for any $j \in [1,\psi]$. 

$\Phi$ has the following geometrical interpretation:
\begin{itemize}
\setlength\itemsep{0.2em}
 \item[(a)] ${\Phi}(x)= \left[\mathbbmtt{1}, \dots, \mathbbmtt{1}  \right]$: $\Vert {\Phi}(x) \Vert\ = \sqrt{t}$ and $\kappa_I(x,x|D)=1$ iff $\Phi_i(x) \ne \mathbf{0}$ for all $i \in [1,t]$.
 \item[(b)]  For point $x$ such that $\exists i \in [1,t], \Phi_i(x) = \mathbf{0}$; then $\Vert {\Phi}(x) \Vert\ < \sqrt{t}$.
 \item[(c)]  If point $x \in \mathbb{R}^d$ falls outside of all hyperspheres in $H_i$ for all $i \in [1,t]$, then it is mapped to the origin of the feature space $\Phi(x) = \left[\mathbf{0},\dots,\mathbf{0}\right]$.
\end{itemize}

\begin{figure}
    \centering
    {\includegraphics[height=.18\textwidth,width=.40\textwidth]{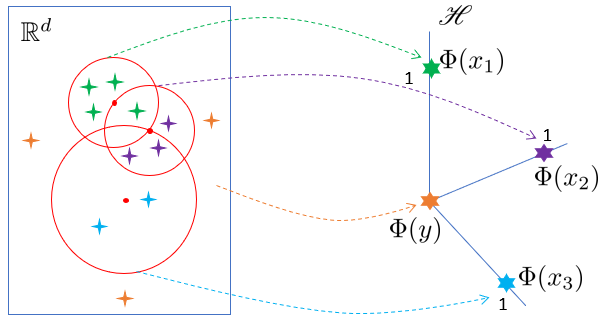}}
  \caption{An illustration of feature map  ${\Phi}$ of Isolation Kernel with one partitioning ($t=1$) of three hyperspheres, each centred at a point (as red dot)  $z \in \mathcal{D}$  where $|\mathcal{D}|=\psi=3$ points are randomly selected from the given dataset $D$.  When a point $x$ falls into an overlapping region, $x$ is regarded to be in the hypersphere whose centre is closer to $x$. }
  \label{fig_example}
\end{figure}

Let $T$ be a set of normal points and $S$ a set of anomalies. 
Let the given dataset $D= T$  which consists of normal points only\footnote{This assumption is for clarity of the exposition only; and $D$ is used to derive Isolation Kernel. In practice, $D$ could contain anomalies but has a minimum impact on $\kappa_I$. See the details later.}.
In the context of anomaly detection, assuming that $\mathcal{P}_T$ is the distribution of normal points $x$ and the largely different $\mathcal{P}_S$ is the distribution of point anomalies $y$. 
The geometrical interpretation gives rise to: (i) point anomalies $y \in S$ are mapped close to the origin of RKHS because  
they are different from normal points $x \in T$---they largely satisfy condition (c) and sometimes (b); and (ii) $\Phi$ of individual normal points $x \in T$ have norm equal or close to $\sqrt{t}$---they largely satisfy condition (a) and sometimes (b). In other words, normal points are mapped to or around  $\left[\mathbbmtt{1}, \dots, \mathbbmtt{1}  \right]$. Note that  $\left[\mathbbmtt{1}, \dots, \mathbbmtt{1}  \right]$ is not a single point in RKHS, but points which have $\Vert {\Phi}(x) \Vert\ = \sqrt{t}$ and $\Phi_i(x) = \mathbbmtt{1}$ for all $i \in [1,t]$.

Figure \ref{fig_example} shows an example mapping $\Phi$ of Isolation Kernel using $\psi=3$ and $t=1$, where all points falling into a particular hypersphere are mapped to the same point in RKHS. Those points which fall outside of all hyperspheres are mapped to the origin of RKHS.

The previous implementations of Isolation Kernel \cite{ting2018IsolationKernel,IsolationKernel-AAAI2019} possess condition (a) only. Their capability to separate anomalies from normal points  by using the norm of $\Phi$ assumes that anomalies will fall on the partitions at the fringes of the data distribution. While this works for many anomalies, this capacity is weaker.

Note that Definition \ref{def:featureMap} is an \textbf{exact feature map} of Isolation Kernel. Re-express Eq \ref{Eqn_IK} using $\Phi$ gives:
  \begin{eqnarray}
\kappa_I(x,y\ |\ D) = \frac{1}{t} \left< \Phi(x|D), \Phi(y|D) \right>
   \label{eqn_IK_feature_map}
\end{eqnarray}

In contrast, existing finite-dimensional feature map derived from a data independent kernel is an \textbf{approximate feature map}, i.e.,
  \begin{eqnarray}
\kappa(x,y) \approx \left< \varphi(x), \varphi(y) \right>
   \label{eqn_approx_feature_map}
\end{eqnarray}

\subsection{Data dependent property of new Isolation Kernel}
The new partitioning mechanism produces large hyperspheres in a sparse region and small hyperspheres in a dense region. 
This yields the following property~\cite{ting2018IsolationKernel}:  \emph{two points in a sparse region are more similar than two points of equal inter-point distance in a dense region.}

Here we provide a theorem for the equivalent property: \textbf{two points, as measured by Isolation Kernel derived in a sparse region, are more similar than the same two points, as measured by Isolation Kernel derived in a dense region.}

\begin{thm}\label{lem_property}
Given two probability distributions $\mathcal{P}_D, \mathcal{P}_{D'} \in \mathbb{P}$ from which points in datasets $D$ and $D'$ are drawn, respectively. Let $\mathcal{E} \subset \mathcal{X}$ be a region such that $\forall_{w\in \mathcal{E}},  \mathcal{P}_D(w)<\mathcal{P}_{D'}(w)$, {\it i.e.}, $D$ is sparser than $D'$ in $\mathcal{E}$. Assume that $\psi=|\mathcal{D}|$ is large such that $\tilde{z}$ is the nearest neighbour of $z$, where $z, \tilde{z}  \in \mathcal{D} \subset D$ in $\mathcal{E}$, under a given metric distance $\ell$ (the same applies to $z', \tilde{z}'  \in \mathcal{D'} \subset D'$ in $\mathcal{E}$.)
Isolation Kernel $\kappa_I$ based on hyperspheres $\theta(z) \in H$ has the property that $\kappa_I( x, y\ |\ D) > \kappa_I( x, y\ |\ D')$ for any point-pair $x,y \in \mathcal{E}$.
\end{thm}
\begin{proof} Let $\ell$ between $x$ and $y$ be $\ell_{xy}$. $x,y\in \theta[z]$ if and only if the nearest neighbour of both $x$ and $y$ is $z$ in $\mathcal{D}$, and   $\ell_{z\tilde{z}}>\max(\ell_{xz},\ell_{yz})$ holds for $\tilde{z}$ the nearest neighbour of $z$ in $\mathcal{D}$. Moreover, the triangular inequality $\ell_{xz}+\ell_{yz}>\ell_{xy}$ holds because $\ell$ is a metric distance. Accordingly, 
\begin{eqnarray}
&&\hspace{-8mm}P(x,y\in \theta[z]\ | \ z\in \mathcal{D} \subset D)\nonumber\\
&&\hspace{-2.8mm}=P(\{\ell_{z\tilde{z}}>\ell_{xz}>\ell_{yz}\}\land\{\ell_{yz}>\ell_{xy}-\ell_{yz}\})+\nonumber\\
&&P(\{\ell_{z\tilde{z}}>\ell_{xz}>\ell_{xy}-\ell_{yz}\}\land\{\ell_{xy}-\ell_{yz}>\ell_{yz}\})+\nonumber\\
&&P(\{\ell_{z\tilde{z}}>\ell_{yz}>\ell_{xz}\}\land\{\ell_{xz}>\ell_{xy}-\ell_{xz}\})+ \label{pxyA0}\\ 
&&P(\{\ell_{z\tilde{z}}>\ell_{yz}>\ell_{xy}-\ell_{xz}\}\land\{\ell_{xy}-\ell_{xz}>\ell_{xz}\})\nonumber\\
&&\hspace{-2.8mm}=2P(\{\ell_{z\tilde{z}}>\ell_{xz}>\ell_{yz}\}\land\{\ell_{yz}>\ell_{xy}-\ell_{yz}\})+\nonumber\\
&&2P(\{\ell_{z\tilde{z}}>\ell_{xz}>\ell_{xy}-\ell_{yz}\}\land\{\ell_{xy}-\ell_{yz}>\ell_{yz}\})\nonumber
\end{eqnarray}
subject to the nearest neighbour $z \in \mathcal{D}$ of both $x$ and $y$. The last equality holds by the symmetry of $\ell_{xz}$ and $\ell_{yz}$.

Given a hypersphere $v(c,\ell_{cz})$ centered at $c \in \mathcal{E}$ and having radius $\ell_{cz}$ equal to the distance from $c$ to its nearest neighbour $z \in \mathcal{D}$, let $\mathcal{P}(u(c,\ell_{cz}))$ be the probability density of probability mass $u(c,\ell_{cz})$ in $v(c,\ell_{cz})$; $u(c,\ell_{cz})=\int_{v(c,\ell_{cz})} \mathcal{P}_D(w)dw$. Note that $u(c,\ell_{cz})$ is strictly monotonic to $\ell_{cz}$ if $v(c,\ell_{cz}) \cap \mathcal{X} \neq \emptyset$, since $\mathcal{P}_D$ is strictly positive in $\mathcal{X}$. Then, the followings are derived.
\begin{eqnarray}
&&\hspace{-10mm}P(\{\ell_{z\tilde{z}}>\ell_{xz}>\ell_{yz}\}\land\{\ell_{yz}>\ell_{xy}-\ell_{yz}\})\nonumber\\
&&\hspace{-4.8mm}=P(\ell_{z\tilde{z}}>\ell_{xz}>\ell_{yz}>\ell_{xy}/2)\nonumber\\
&&\hspace{-4.8mm}=\int_{u(z,\ell_{xy}/2)}^1\mathcal{P}(u(z,\ell_{z\tilde{z}}))\int_{u(x,\ell_{xy}/2)}^{u(x,\ell_{z\tilde{z}})}\mathcal{P}(u(x,\ell_{xz})) \times\nonumber\\
&&\hspace{-2mm}\int_{u(y,\ell_{xy}/2)}^{u(y,\ell_{xz})}\mathcal{P}(u(y,\ell_{yz}))du(y,\ell_{yz})du(x,\ell_{xz})du(z,\ell_{z\tilde{z}}) \label{pxyA1}\\
&&\hspace{-4.8mm}\approx \int_{u(z,\ell_{xy}/2)}^{u(z,\hat{\ell}_{z\mathcal{E}})}\mathcal{P}(u(z,\ell_{z\tilde{z}}))\int_{u(x,\ell_{xy}/2)}^{u(x,\ell_{z\tilde{z}})}\mathcal{P}(u(x,\ell_{xz})) \times\nonumber\\
&&\hspace{-2mm}\int_{u(y,\ell_{xy}/2)}^{u(y,\ell_{xz})}\mathcal{P}(u(y,\ell_{yz}))du(y,\ell_{yz} du(x,\ell_{xz})du(z,\ell_{z\tilde{z}}),\nonumber
\end{eqnarray}
\begin{eqnarray}
&&\hspace{-10mm}P(\{\ell_{z\tilde{z}}>\ell_{xz}>\ell_{xy}-\ell_{yz}\}\land\{\ell_{xy}-\ell_{yz}>\ell_{yz}\})\nonumber\\
&&\hspace{-4.8mm}=P(\{\ell_{z\tilde{z}}>\ell_{xz}>\ell_{xy}-\ell_{yz}\}\land\{\ell_{xy}/2)>\ell_{yz}\})\nonumber\\
&&\hspace{-4.8mm}=\int_{u(z,\ell_{xy}/2)}^1 \mathcal{P}(u(z,\ell_{z\tilde{z}}))\int_{0}^{u(y,\ell_{xy}/2)}\mathcal{P}(u(y,\ell_{yz})) \times\nonumber\\
&&\hspace{-2mm}\int_{u(x,\ell_{xy}-\ell_{yz})}^{u(x,\ell_{z\tilde{z}})}\mathcal{P}(u(x,\ell_{xz}))du(x,\ell_{xz})du(y,\ell_{yz})du(z,\ell_{z\tilde{z}}) \label{pxyA2}\\
&&\hspace{-4.8mm} \approx \int_{u(z,\ell_{xy}/2)}^{u(z,\hat{\ell}_{z\mathcal{E}})}\mathcal{P}(u(z,\ell_{z\tilde{z}}))\int_{0}^{u(y,\ell_{xy}/2)}\mathcal{P}(u(y,\ell_{yz})) \times\nonumber\\
&&\hspace{-2mm}\int_{u(x,\ell_{xy}-\ell_{yz})}^{u(x,\ell_{z\tilde{z}})}\mathcal{P}(u(x,\ell_{xz}))du(x,\ell_{xz})du(y,\ell_{yz})du(z,\ell_{z\tilde{z}}),\nonumber
\end{eqnarray}
where $\hat{\ell}_{z\mathcal{E}}=\sup_{v(z,\ell_{z\tilde{z}}) \subseteq \mathcal{E}} \ell_{z\tilde{z}}$. The approximate equality holds by the assumption in the theorem which implies that the integral from $u(z,\hat{\ell}_{z\mathcal{E}})$ to $1$ for $u(z,\ell_{z\tilde{z}})$ is negligible. The same argument applied to $P(x,y\in \theta[z']\ | \ z'\in \mathcal{D}' \subset D')$ which derives the identical result.

\cite{Fukunaga-Book1990} provided the expressions of $\mathcal{P}(u(c,\ell_{cz}))$ and $\mathcal{P}(u(z,\ell_{z\tilde{z}}))$ as
\begin{eqnarray*}
\mathcal{P}(u(c,\ell_{cz})) &=& \psi(1-u(c,\ell_{cz}))^{\psi-1},\\
\mathcal{P}(u(z,\ell_{z\tilde{z}})) &=& (\psi-1)(1-u(z,\ell_{z\tilde{z}}))^{\psi-2}.
\end{eqnarray*}
With these expressions and the definition of $u(c,\ell_{cz})$, both $\mathcal{P}(u(c,\ell_{cz}))$ and $\mathcal{P}(u(z,\ell_{z\tilde{z}}))$ are lower if $\mathcal{P}_D(z)$ becomes higher. Accordingly,
\[P(x,y\in \theta[z]\ | \ z\in \mathcal{D} \subset D) > P(x,y\in \theta[z']\ | \ z'\in \mathcal{D}' \subset D')\]
holds by the fact $\forall_{w\in \mathcal{E}},  \mathcal{P}_D(w)<\mathcal{P}_{D'}(w)$, Eq.~\ref{pxyA0}, Eq.~\ref{pxyA1} and Eq.~\ref{pxyA2}.
This result and Definition~\ref{IKernel} prove the theorem. 
\end{proof}

Theorem \ref{lem_property} is  further evidence that the data dependent property of Isolation Kernel only requires that the isolation mechanism produces large partitions in sparse region and small partitions in dense region, regardless of the actual space partitioning mechanism. 

We introduce Isolation Distributional Kernel, its theoretical analysis and data dependent property in the next section. 

\section{Isolation Distributional Kernel}
\label{sec_IDK}
Given the feature map $\Phi$ (defined in Definition \ref{def:featureMap}) and Eq~\ref{eqn_IK_feature_map}, the empirical estimation of kernel mean embedding can be expressed based on the feature map of Isolation Kernel $\kappa_I(x,y)$.

\begin{definition}\label{IDKernel}
Isolation Distributional Kernel of two distributions $\mathcal{P}_S$ and $\mathcal{P}_T$ is given as: 
\begin{eqnarray}
\widehat{\mathcal{K}}_I(\mathcal{P}_S,\mathcal{P}_T\ |\ D) 
 & = & \frac{1}{t|S||T|}\sum_{x\in S} \sum_{y\in T} \Phi(x|D)^\top\Phi(y|D)\nonumber\\
 & = & \frac{1}{t} \left< \widehat{\Phi}(\mathcal{P}_S|D), \widehat{\Phi}(\mathcal{P}_T|D) \right> \label{eqn_IDK}
\end{eqnarray}
where $\widehat{\Phi}(\mathcal{P}_S|D) = \frac{1}{|S|} \sum_{x \in S} \Phi(x|D)$ is the empirical feature map of the kernel mean embedding. 
\end{definition}
Hereafter, `$|D$' is omitted when the context is clear.

Condition (a) wrt $\Vert \Phi(x) \Vert$ in Section \ref{sec_properties} leads to $\Vert \widehat{\Phi}(x) \Vert\ = \sqrt{t}$; and similarly $0 \leq \Vert \widehat{\Phi}(x) \Vert\ < \sqrt{t}$ holds under conditions (b) and (c) in Section~\ref{sec_properties}. Thus, $\left< \widehat{\Phi}(\mathcal{P}_S), \widehat{\Phi}(\mathcal{P}_T) \right> \  \in [0,t]$ {\it i.e.}, $\widehat{\mathcal{K}}_I(\mathcal{P}_S,\mathcal{P}_T) \in [0,1]$.

The key advantages of IDK over existing kernel mean embedding \cite{EMK_Bo-NIPS2009,SupportMeasureMchines-NIPS2012,Sutherland-Thesis2016} are:
(i) $\Phi$ is an exact and finite-dimensional feature map of a data dependent point kernel; whereas $\varphi$ in Eq~\ref{eqn_MMK_approximation} is an approximate feature map of a data independent point kernel.
(ii) The distributional characterisation of $\widehat{\Phi}(\mathcal{P}_T)$ is derived from $\Phi$'s adaptability to local density in $T$; whereas the distributional characterisation of $\widehat{\varphi}(\mathcal{P}_T)$ lacks such adaptability because $\varphi$ of Gaussian kernel is data independent. This is despite the fact that both Isolation Kernel and Gaussian kernel are a characteristic kernel (see the next section.)

\subsection{Theoretical Analysis: Is Isolation Kernel a characteristic kernel?}
\label{sec_characteristicKernel}

As defined in subsection~\ref{sec_kernel_mean_embedding}, we consider $\mathcal{P}_S, \mathcal{P}_T \in \mathbb{P}$, where $\mathbb{P}$ is a set of probability distributions on $\mathbb{R}^d$ which are admissible but strictly positive  on $\mathcal{X}$ and strictly zero on $\overline{\mathcal{X}}$. This implies that no data points exist outside of $\mathcal{X}$. Thus we limit our analysis to the property of the kernel on $\mathcal{X}$. 
A positive definite kernel $\kappa$ is a characteristic kernel if its kernel mean map $\widehat{\Phi}: \mathbb{P} \rightarrow \mathscr{H}$ is injective, {\it i.e.}, $\Vert \widehat{\Phi}(\mathcal{P}_S) - \widehat{\Phi}(\mathcal{P}_T) \Vert_\mathscr{H} = 0$ if and only if $\mathcal{P}_S = \mathcal{P}_T$~\cite{KernelMeanEmbedding2017}. If the kernel $\kappa$ is non-characteristic,
two different distributions $\mathcal{P}_S \ne \mathcal{P}_T$ may be mapped to the same $\widehat{\Phi}(\mathcal{P}_S)=\widehat{\Phi}(\mathcal{P}_T)$. 

Isolation Kernel derived from the partitioning $H_i$ can be interpreted that $\mathcal{X}$ is packed by hyperspheres having random sizes. This is called random-close packing~\cite{SoftMatter2014}. Previous studies 
revealed that the upper bound of the rate of the packed space in a 3-dimensional space is almost $64\%$ for any random-close packing~\cite{Nature2008,SoftMatter2014}. The packing rates for the higher dimensions are known to be far less than $100\%$ for any distribution of sizes of hyperspheres. This implies that the $(\psi+1)$-th partition, which is not covered by any hyperspheres, always has nonzero volume. 

Let $R \subset \mathcal{X}$ and $\overline{R} = \mathcal{X} \setminus R$ be regions such that $\forall x \in R$, $\mathcal{P}_S(x) \neq \mathcal{P}_T(x)$ and $\forall x \in \overline{R}$, $\mathcal{P}_S(x) = \mathcal{P}_T(x)$. 
From the fact that $\int_{\overline{R}} (\mathcal{P}_S(x)-\mathcal{P}_T(x))dx = 0$ and $\int_\mathcal{X} \mathcal{P}_S(x)dx = \int_\mathcal{X} \mathcal{P}_T(x)dx =1$,
\begin{equation}
\int_R (\mathcal{P}_S(x)-\mathcal{P}_T(x))dx = 0 
\label{eq_RPdiff}
\end{equation}
is deduced. In conjunction with this relation and the fact $\forall x \in R$, $\mathcal{P}_S(x) \neq \mathcal{P}_T(x)$, there exists at least one $R' \subset R$ such that 
\begin{equation}
\int_{R'}(\mathcal{P}_S(x)-\mathcal{P}_T(x))dx \neq 0.
\label{eq_Pdiff}
\end{equation}
This requires $R$ to contain at least two distinct points in $\mathcal{X}$.

Accordingly, a partitioning $H_i$ of Isolation Kernel satisfies one of the following two mutually exclusive cases:

\begin{itemize}[\itemindent=2em]
\setlength\itemsep{0em}
    \item[Case\,1:] $\exists \theta \in H_i, \theta \supseteq R$. From Eq~\ref{eq_RPdiff}, the probability of $x \sim \mathcal{P}_S$ falling into $\theta$ and that of $x \sim \mathcal{P}_T$ falling into $\theta$ are identical. Thus, the difference between $\mathcal{P}_S$ and $\mathcal{P}_T$ does not produce any difference between $\Phi(x \sim \mathcal{P}_S)$ and $\Phi(x \sim \mathcal{P}_T)$ in expectation.
    \item[Case\,2:] $\exists \theta \in H_i, \theta \cap R \ne \emptyset$ and $\theta \not\supseteq R$.  If $\theta \cap R$ is one of $R'$ satisfying Eq \ref{eq_Pdiff}, then $\Phi(x \sim \mathcal{P}_S)$ and $\Phi(x \sim \mathcal{P}_T)$ are different in expectation.
\end{itemize}

These observations give rise to the following theorem:

\begin{thm}\label{lemIKt}
The kernel mean map of the new Isolation Kernel (generated from $D$) $\ \widehat{\Phi}: \mathbb{P} \rightarrow \mathscr{H}$ is characteristic in $\mathcal{X}$ with probability $1$ in the limit of $t \rightarrow \infty$, for $\psi, t \ll |D|$.
\end{thm}
\begin{proof}
To define $H$, points in $\mathcal{D} \subset D$ are drawn from $\mathcal{P}_D(x)$ which is strictly positive on $\mathcal{X}$, {\it i.e.}, $\forall X \subseteq \mathcal{X} \mbox{ s.t. } X \neq \emptyset, \mathcal{P}_D(X)>0$. 
This implies that any partitioning $H$ of $\mathcal{X}$, created by $\mathcal{D}$, has non-zero probability, since the points in $\mathcal{D}$ can be anywhere in $\mathcal{X}$ with non-zero probability. Also recall that $\psi = |\mathcal{D}| = 2$ is the minimum sample size required to construct the hyperspheres in $H$ (see Section~\ref{sec_IK_implementation}.)  Let's call this: the property of $H$.\\
Due to the property of $H$ and $\psi \ge 2$, there exists $H_i \in \{H_1,\dots,H_t\}$ and $\theta_j \in H_i$ with probability $1$ such that $x \in \theta_j$ and $y \notin \theta_j$ for any mutually distinct points $x$ and $y$ in $\mathcal{X}$, as $t \rightarrow \infty$. This implies that, as $t \rightarrow \infty$, there exists $\Phi_{ij}(x)$ for any $x \in D$ with probability $1$ such that $\Phi_{ij}(x)=1$ and $\Phi_{ij}(y)=0, \forall y \in D, y \neq x$. Then, the Gram matrix of Isolation Kernel is full rank, because the feature maps $\Phi(x)$ for all points $x \in D$ are mutually independent. Accordingly, Isolation Kernel is a positive definite kernel with probability $1$ in the limit of $t \rightarrow \infty$.\\
Because of the property of $H$ and the fact that all $\theta \in H_i$ including the $\psi+1$-th partition have non-zero volumes for any $\psi \geq 2$, the probability of Case 1 is not zero. In addition, because $R$ contains at least two distinct points in $\mathcal{X}$, the probability of Case 2 is not zero for any $\psi \geq 2$.
These facts yield $0<p<1$, where $p$ is the probability of an event that $H_i$ satisfies Case\,1 but not Case\,2. 
Since $\psi, t \ll |D|$, $H_i$ are almost independently sampled over $i=1,\dots,t$, and the probability of the event occurring over all $i=1,\dots,t$ is $p^t$. If $t \rightarrow \infty$, then $p^t \rightarrow 0$. This implies that Isolation Kernel is injective with probability $1$ in the limit of $t \rightarrow \infty$.\\
Both the positive definiteness and the injectivity imply that Isolation Kernel is a characteristic kernel.
\end{proof}
\vspace{-2mm}

Some data independent kernels such as Gaussian kernel are characteristic \cite{KernelMeanEmbedding2017}. Because an empirical estimation uses a finite dataset, their kernel mean maps that ensure injectivity~\cite{Sriperumbudur:2010} are as good as that using Isolation Kernel with large $t$.
To use the kernel mean map for anomaly detection, a point kernel deriving the kernel mean map must be characteristic. Otherwise, anomalies of $\mathcal{P}_S$ may not be properly separated from normal points of $\mathcal{P}_T$ because some anomalies and normal points may be mapped to an identical point $\widehat{\Phi}(\mathcal{P}_T)=\widehat{\Phi}(\mathcal{P}_S)$. 
As we show in the experiment section, Isolation Kernel using $t=100$ is sufficient to produce better result than Gaussian kernel in kernel mean embedding for anomaly detection.

\subsection{Data dependent property of IDK}

Following Definitions \ref{IDKernel}, IDK of two distributions $\mathcal{P}_S$ and $\mathcal{P}_T$ can be redefined as the expected probability over the probability distribution on all partitionings $H\in\mathds{H}_\psi(D)$ that a randomly selected point-pair $x\in S$ and $y\in T$ falls into the same isolation partition $\theta[z] \in H$, where $z \in \mathcal{D} \subset D$:
\[\widehat{\mathcal{K}}_I(\mathcal{P}_S,\mathcal{P}_T )=\mathbb{E}_{\mathds{H}_\psi(D)}\left[\mathbb{I}(x,y\in \theta[z]\ |\ \theta[z] \in H;x\in S,y\in T)\right]
\]

If the supports of both $\mathcal{P}_S$ and $\mathcal{P}_T$ are included in $\mathcal{E}$ and $\forall_{w\in \mathcal{E}},$  $\mathcal{P}_D(w)<\mathcal{P}_{D'}(w)$ holds, the above expression leads to the following proposition, since every point-pair from $S$ and $T$ follows Theorem~\ref{lem_property}. 

\begin{prop} Under the conditions on $\mathcal{P}_D$ \& $\mathcal{P}_{D'}$, $D$ \& $D'$ and $\mathcal{E}$ of Theorem~\ref{lem_property}, given two distribution-pairs $\mathcal{P}_S, \mathcal{P}_T \in \mathbb{P}$ where the supports of both $\mathcal{P}_S$ and $\mathcal{P}_T$ are in $\mathcal{E}$, the IDK $\widehat{\mathcal{K}}_I$  based on hyperspheres $\theta(z) \in H$ has the property that \[\widehat{\mathcal{K}}_I(\mathcal{P}_S,\mathcal{P}_T | D) > \widehat{\mathcal{K}}_I(\mathcal{P}_{S},\mathcal{P}_{T} | D').\]
\end{prop}

In other words, the data dependent property of Isolation Kernel leads directly to the data dependent property of IDK:\\ \textbf{Two distributions, as measured by IDK derived in sparse region, are more similar than the same two distributions, as measured by IDK derived in dense region}.

We organize the rest of the paper as follows. For point anomaly detection: the proposed method, its experiments and relation to isolation-based anomaly detectors are provided in Sections \ref{sec_anomaly_detection}, \ref{sec_experiments} and \ref{sec_relation}, respectively. For group anomaly detection: we present the background, the proposed method, conceptual comparison, time complexities and the experiments in the next five sections; and we conclude in the last section.

\section{Proposed Method for\\ Point Anomaly Detection}
\label{sec_anomaly_detection}
Let the kernel mean embedding $\widehat{\mathcal{K}}$ has $\widehat{\Phi}(\mathcal{P}_D)$ be a kernel mean mapped point of a distribution $\mathcal{P}_D$ of a dataset $D$. 
$\widehat{\mathcal{K}}$ measures the similarity of a Dirac measure $\delta(x)$ of a point $x$ wrt $\mathcal{P}_D$ as follows:

\begin{itemize}
    \item If $x \sim \mathcal{P}_D$, $\widehat{\mathcal{K}}(\delta(x),\mathcal{P}_D)$ is large, which can be interpreted as $x$ is likely to be part of $\mathcal{P}_D$.

    \item If $y \not\sim \mathcal{P}_D$, $\widehat{\mathcal{K}}(\delta(y),\mathcal{P}_D)$ is small, which can be interpreted as $y$ is not likely to be part of $\mathcal{P}_D$.
\end{itemize}

With this interpretation, $y \not\sim \mathcal{P}_D$ can be naturally used to identify anomalies in $D$. The definition of anomaly is thus:

`{\em Given a similarity measure $\widehat{\mathcal{K}}$ of two distributions, an anomaly is an observation whose Dirac measure $\delta$ has a low similarity with respect to the distribution from which a reference dataset is generated.}' 

This is an operational definition, based on distributional kernel $\widehat{\mathcal{K}}$, of the one given by Hawkins (1980):

`An outlier is an observation which deviates so much from the other observations as to arouse suspicions that it was generated by a different mechanism.'

Let $\mathcal{P}_T$ and $\mathcal{P}_S$ be the distributions of normal points and point anomalies, respectively. We assume that $\forall x \in T, x \sim \mathcal{P}_T$; and $S$ consists of subsets of anomalies $\mathsf{S}_i$ from  mutually distinct distributions, i.e.,  $S = \cup_{i=1}^{m} \mathsf{S}_i$, where $\mathcal{P}_{\mathsf{S}_i}=\delta(x_i)$ is the distribution of an anomaly $x_i$ represented as a Dirac measure $\delta$, and $\mathsf{S}_i=\{x_i\}$. 

Note that in the unsupervised learning context, only the dataset $D = S \cup T$ is given without information about $S$ and $T$, where the majority of the points in $D$ are from $\mathcal{P}_T$; but $D$ also contains few points from $\mathcal{P}_S$. Because $|S| \ll |T|$, $\mathcal{P}_D \approx \mathcal{P}_T$. The kernel mean map of $\mathcal{P}_T$ is empirically estimated from $D$, i.e.,  $\widehat{\Phi}(\mathcal{P}_T) \approx \widehat{\Phi}(\mathcal{P}_D) = \frac{1}{|D|} \sum_{x \in D} \Phi(x)$ is thus robust against influences from $\mathcal{P}_S$ because the few points in $S$ have significantly lower weights than the many points in $T$. Therefore, $\widehat{\Phi}(\mathcal{P}_T) \approx \widehat{\Phi}(\mathcal{P}_D)$ is in a region distant from those of  $\widehat{\Phi}(\mathcal{P}_{\mathsf{S}_i})$. In essence, \emph{IDK derived from $D$ is robust against contamination of anomalies in $D$}.


\textbf{Point anomaly detector $\widehat{\mathcal{K}}$}: Given $D$, the proposed detector is one which summarizes the entire dataset $D$ into one mapped point $\widehat{\Phi}(\mathcal{P}_D)$. To detect point anomalies, map each $x \in D$ to $\widehat{\Phi}(\delta(x))$;
and compute its similarity w.r.t. $\widehat{\Phi}(\mathcal{P}_D)$, i.e., $\left< \widehat{\Phi}(\delta(x)), \widehat{\Phi}(\mathcal{P}_D) \right>$. Then sort all computed similarities to rank all points in $D$. Anomalies are those points which are least similar to $\widehat{\Phi}(\mathcal{P}_D)$. The anomaly detector due to IDK is computed using  Eq~\ref{eqn_IDK}, denoted as $\widehat{\mathcal{K}}_I$; and the detectors using Gaussian kernel and its Nystr\"{o}m approximation are computed using Eq \ref{eqn_mmk} and Eq \ref{eqn_MMK_approximation}, denoted as $\widehat{\mathcal{K}}_G$ and $\widehat{\mathcal{K}}_{NG}$.

\begin{figure}[t]
\includegraphics[height=.16\textwidth,width=.44\textwidth]{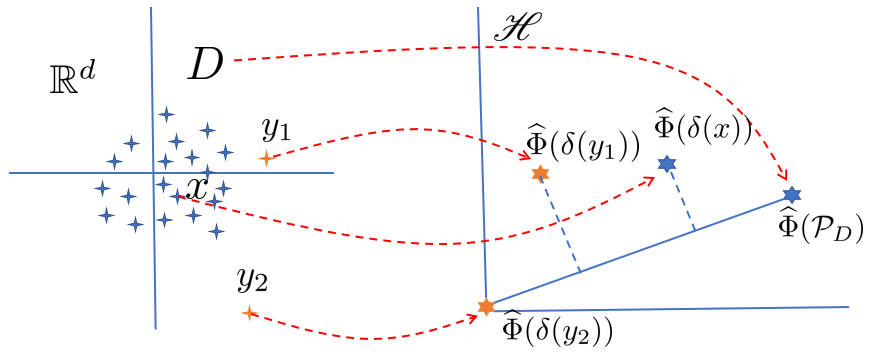}
    \hspace{15mm}
  \vspace{-6mm}
\caption{An illustration of kernel mean mapping $\widehat{\Phi}$ used in $\widehat{\mathcal{K}_I}$ for a dataset $D$, containing anomalies $y_1$ \& $y_2$. $\widehat{\Phi}$ maps from $\mathbb{R}^d$ of $D$ \& individual points to points in $\mathscr{H}$.}
  \label{fig_example2}
  \vspace{-2mm}
\end{figure}

Figure \ref{fig_example2} shows an example mapping $\widehat{\Phi}$ of a normal point $x$, two anomalies $y_1$ and $y_2$ as well as $D$ from $\mathbb{R}^d$ to $\mathscr{H}$. 
The global anomaly $y_2$ is mapped to the origin of $\mathscr{H}$; and $y_1$ which is just outside the fringe of a normal cluster is mapped to a position where
$\widehat{\mathcal{K}}_I(\delta(y_1), \mathcal{P}_D)$
is closer to the origin than that of normal points.

\section{Experiments: point anomaly detection}
\label{sec_experiments}

The detection accuracy of an anomaly detector is measured in terms of AUC (Area under ROC curve). As all the anomaly detectors are unsupervised learners, all models are trained with unlabeled training sets. Only after the models have made predictions, ground truth labels are used to compute the AUC for each dataset. We report the runtime of each detector in terms of CPU seconds.

We compare kernel based anomaly detectors $\widehat{\mathcal{K}}_I$, $\widehat{\mathcal{K}}_G$ and $\widehat{\mathcal{K}}_{NG}$ (using the Nystr\"{o}m method \cite{musco2017recursive} to produce an approximate feature map) and 
OCSVM \cite{OCSVM2001} which employs Gaussian kernel\footnote{Another type of `kernel' based anomaly detectors relies on a kernel density estimator (KDE) as the core operation~\cite{KDEOS,Local_kernel_density}.  These methods have $O(n^2)$ time complexity.
We compare our method with kernel methods that employ a kernel as similarity measure like OCSVM \cite{OCSVM2001}, which are not a density based method that relies on KDE or a density estimator \cite{FastLOF-TKDE2016}. A comparison between this type of anomaly detectors and isolation-based anomaly detectors is given in \cite{iNNE}.}.
Parameter settings used in the experiments: $\widehat{\mathcal{K}}_I$  uses $t=100$; and $\psi$ is searched over $\psi\in\{2^m\ |\ m=1,2...12\}$. For all methods using Gaussian kernel, the bandwidth is searched over $\{2^m\ |\ m=-5,\dots 5\}$. 
The sample size of the Nystr\"{o}m method is set as $\sqrt{n}$ which is also equal to the number of features. All datasets are normalized to $[0,1]$ in the preprocessing. 

All codes are written in Matlab running on Matlab 2018a. The experiments are performed on a machine with 4$\times$2.60 GHz CPUs and 8GB main memory.

\begin{table}[t]
		\centering
		\caption{Results of kernel based anomaly detectors (AUC). }

		\label{tbl_point_anomalies}
		\hspace{-2mm}
		\begin{tabular}{lrrrcccc}
			\toprule
			Dataset &\#Inst&\#Ano&\#Attr& $\widehat{\mathcal{K}}_I$ &$\widehat{\mathcal{K}}_G$&$\widehat{\mathcal{K}}_{NG}$&OCSVM  \\ \midrule
			speech&3,686&61&400 &0.76&0.46&0.47&0.65\\
			EEG\_eye&8,422&165&14&0.88&0.55&0.47&0.54\\
			PenDigits&9,868&20&17&0.98&0.98&0.95&0.98\\
			MNIST\_230&12,117&10&784&0.98&0.97&0.97&0.96\\
			MNIST\_479&12,139&50&784&0.86&0.69&0.60&0.59\\
			mammograg &11,183&260&6&0.88&0.85&0.86&0.84\\
			electron &37,199&700&50&0.80&0.65&0.57&0.64\\
			shuttle&49,097&3,511&9&0.98&0.98&0.99&0.98\\
			ALOI &50,000&1,508&27 &0.82&0.60&0.54&0.53\\
			muon &94,066&500&50&0.82&0.63&0.55&0.75\\
			smtp&95,156&30&3&0.95&0.81&0.77&0.91\\
			IoT\_botnet&213,814&1,000&115&0.99&0.83&0.68&0.93\\
			ForestCover&286,048&2,747&10&0.97&0.85&0.86&0.96\\
			http&567,497&2,211&3&0.99&0.99&0.98&0.97\\
			\midrule
			{Average rank} & & & & 1.25&2.64&3.18&2.92
\\
			\bottomrule
		\end{tabular}
	\end{table}

\subsection{Evaluation of kernel based  detectors}
	
Table \ref{tbl_point_anomalies} shows that $\widehat{\mathcal{K}}_I$ 
	is the best anomaly detector among the four detectors.
The huge difference in AUC between $\widehat{\mathcal{K}}_I$ and $\widehat{\mathcal{K}}_G$ (e.g., speech, MNIST\_479 \& ALOI) shows the superiority of Isolation Kernel over data independent Gaussian kernel. As expected, $\widehat{\mathcal{K}}_{NG}$ performed worse than $\widehat{\mathcal{K}}_{G}$ in general because it uses an approximate feature map. 
A Friedman-Nemenyi test \cite{NemenyiTest-2006} in Figure~\ref{top4Nemenyi}  shows that $\widehat{\mathcal{K}}_I$ is significantly better than the other algorithms. 

Table \ref{tbl_runtime-pointanomaly} shows that $\widehat{\mathcal{K}}_I$ has short testing and training times. In contrast, $\widehat{\mathcal{K}}_G$ has the longest testing time though it has the shortest (zero) training time; and OCSVM has the longest training time. This is because its operations are based on points without feature map ($\widehat{\mathcal{K}}_G$ uses Eq~\ref{eqn_mmk}).  While $\widehat{\mathcal{K}}_{NG}$ (uses Eq~\ref{eqn_MMK_approximation}) has significantly reduced the testing time of $\widehat{\mathcal{K}}_G$, it still has  longer training time than $\widehat{\mathcal{K}}_I$ because of the Nystr\"{o}m process. 

Time complexities: Given two distributions which generates two sets, where each set has $n$ points, $\widehat{\mathcal{K}}_G$ takes $O(n^2)$. For $\widehat{\mathcal{K}}_I$, the preprocessing $\widehat{\Phi}(\mathcal{P}_S)$ of a set $S$ of $n$ points takes $O(nt\psi)$ and needs to be completed once only. $\widehat{\mathcal{K}}_I$ takes $O(t\psi)$ only to compute the similarity between two sets. Thus, the overall time complexity is $O(nt\psi)$.
For large datasets, $t\psi \ll n$ accounts for the huge difference in testing times between the two methods we observed in Table \ref{tbl_runtime-pointanomaly}. The time complexity of OCSVM in LibSVM is $O(n^3)$.

\begin{figure}[t]
\centering
    \includegraphics[height=.08\textwidth,width=.28\textwidth]{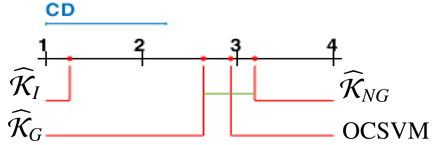}
  \caption{Friedman-Nemenyi test for anomaly detectors at significance level 0.05. If two algorithms
are connected by a CD (critical difference) line, then
there is no significant difference between them.}
  \label{top4Nemenyi}
\end{figure}
	
\begin{table}[t]
		\centering
		\caption{Runtime comparison (in CPU seconds) on  http. }
		\label{tbl_runtime-pointanomaly}
		\begin{tabular}{lrrrr|rr}
			\toprule
			& $\widehat{\mathcal{K}}_I$&$\widehat{\mathcal{K}}_G$ &$\widehat{\mathcal{K}}_{NG}$ & OCSVM &iNNE &iForest\\ \midrule
				train &31 & 0&106&13738&15&6\\
				test &2&9689&1&1964&1&18\\
			\bottomrule
		\end{tabular}

	\end{table}

\subsection{OCSVM fails to detect local anomalies}

Here we examine the abilities of $\widehat{\mathcal{K}}_I$ 
and OCSVM to detect local anomalies  and clustered anomalies. The former is the type of anomalies located in close proximity to normal clusters;
and the latter is the type of anomalies which formed a small cluster located far from all normal clusters. 

Figure \ref{fig_local+clustered_anomalies} shows the distributions of similarities of $\widehat{\mathcal{K}}_I$ and OCSVM on an one-dimensional dataset having three normal clusters of Gaussian distributions (of different variances) and  a small group of clustered anomalies on the right. Note that OCSVM is unable to detect all anomalies located in close proximity to all three clusters, as all of these points have the same (or almost the same) similarity to points at the centers of these clusters. This outcome is similar to that using the density distribution (estimated by KDE) to detect anomalies because the densities of these points are not significantly lower than those of the peaks of low density clusters. 
In addition, the clustered anomalies have higher similarities, as measured by OCSVM, than many anomalies at either fringes of the normal clusters. This means that the clustered anomalies are not included in the top-ranked anomalies.

If local anomalies are defined as points having similarities between 0.25 and 0.75, then the distribution of similarity of $\widehat{\mathcal{K}}_I$ shows that it detects all local anomalies surrounding all three normal clusters; and regards the clustered anomalies as global anomalies (having similarity $< 0.25$). In contrast, OCSVM fails to detect many of the local anomalies detected by $\widehat{\mathcal{K}}_I$; and the clustered anomalies are regarded as local anomalies by OCSVM.

\begin{figure}
    \centering
    {\includegraphics[height=.2\textwidth,width=.38\textwidth]{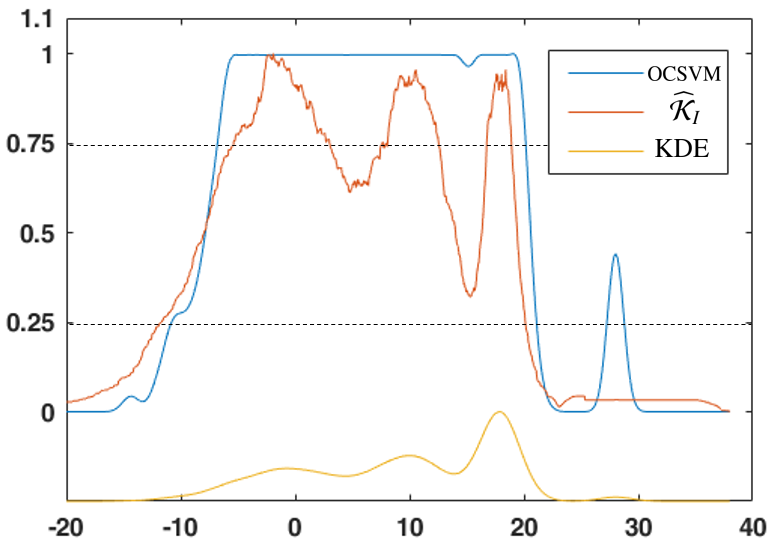}}
    \hspace{-2mm}
  \caption{An one-dimensional dataset having three normal clusters of Gaussian distributions and one anomaly cluster (the small cluster on the right). They have a total of 1500+20 points. The bottom line shows the density distribution as estimated by a kernel density estimator (KDE) using Gaussian kernel (its scale is not shown on the y-axis.) The distributions of scores/similarities of OCSVM and $\widehat{\mathcal{K}}_I$ (scales in y-axis) are shown. The scores of OCSVM have been inversed and rescaled to $[0,1]$ to be comparable to similarity.}
  \label{fig_local+clustered_anomalies}
\end{figure}

\subsection{Stability analysis}

This section provides an analysis to examine the stability of the scores of $\widehat{\mathcal{K}}_{I}$. Because $\widehat{\mathcal{K}}_{I}$ relies on random partitionings, it is important to determine how stable $\widehat{\mathcal{K}}_{I}$ is in different trials using different random seeds.
Figure \ref{boxplot} shows a boxplot of the scores produced by $\widehat{\mathcal{K}}_{I}$ over 10 trials. It shows that $\widehat{\mathcal{K}}_{I}$ becomes more stable as $t$ increases.

\begin{figure}
    \centering
    {\includegraphics[height=.22\textwidth,width=.30\textwidth]{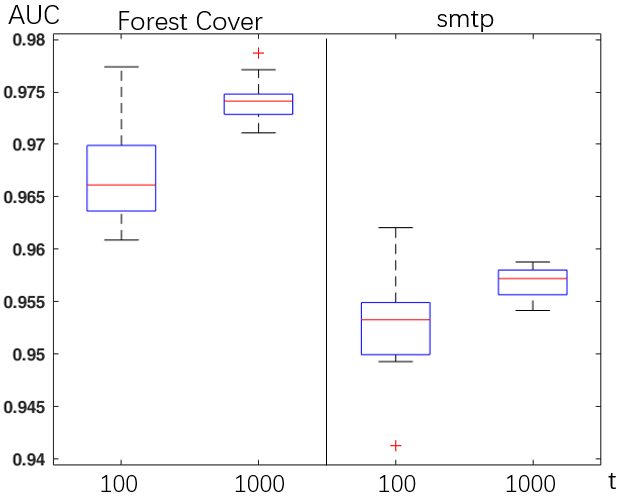}}
  \hspace{-2mm}
  \caption{Boxplot of 10 runs of $\widehat{\mathcal{K}}_{I}$ on ForestCover and smtp. 
  This result shows the increased stability of  $\widehat{\mathcal{K}}_{I}$ predictions as $t$ increases from 100 to 1000.}
\label{boxplot}
\hspace{-2mm}
\end{figure}

\subsection{Robust against contamination of anomalies in the training set}

This section examines how robust an anomaly detector is against contamination anomalies in the training set. We use a `cut-down' version of $\widehat{\mathcal{K}}_{I}$ which does not employ IDK, but Isolation Kernel only, i.e., $\Vert \Phi(x) \Vert$. 
$\Vert \Phi(x) \Vert$ is the base line in examining the robustness of $\widehat{\mathcal{K}}_{I}$ because $\Phi(x)$ is the basis in computing $\widehat{\mathcal{K}}_{I}$ (see Equation \ref{eqn_IDK}). 

Figures~\ref{cover_stable} 
shows an example comparisons between $\Vert \Phi(x) \Vert$ and $\widehat{\mathcal{K}}_I$ on the ForestCover dataset, where exactly the same hyperspheres are used in both detectors. The results show that $\Vert \Phi(x) \Vert$ is very unstable in a dataset with high contamination of anomalies (high $\gamma$). Despite using exactly the same $\Vert \Phi(x) \Vert$, $\widehat{\mathcal{K}}_I$ is very stable for all $\gamma$'s. This is a result of the distributional characterisation of the entire dataset, described in Section \ref{sec_anomaly_detection}.

\begin{figure}
    \centering  
    {\includegraphics[height=.20\textwidth,width=.38\textwidth]{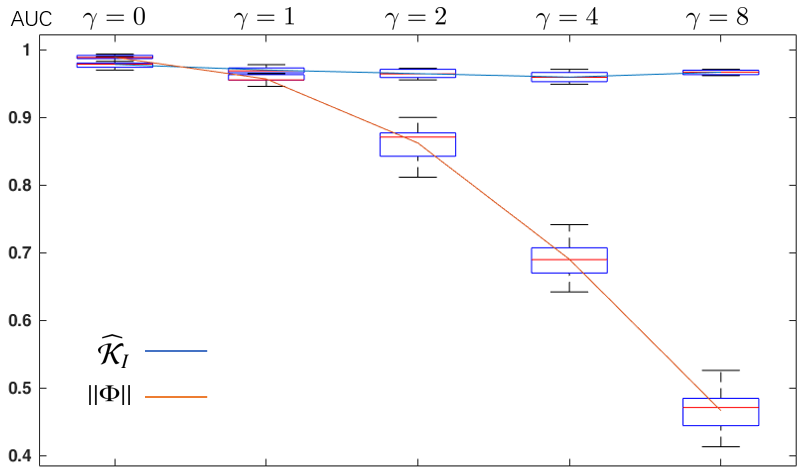}}
  \caption{Boxplot of 10 runs of each of $\widehat{\mathcal{K}}_{I}$ and $\Vert \Phi(x) \Vert$ on ForestCover.  The original anomaly ratio $r$ is the ratio of the number of anomalies and the number of normal points in the given dataset. $\gamma \times r$ is used in the experiment to increase/decrease the anomalies in the given dataset. $\gamma=1$ when the given dataset is used without modification; $\gamma=0$ when no anomalies are used in the training process. $\gamma>1$ has an increasingly higher chance of including anomalies in the training process. }
\label{cover_stable}
\end{figure}


\section{Relation to Isolation-based anomaly detectors}
\label{sec_relation}

Both $\widehat{\mathcal{K}}_I$ and the existing isolation-based detector iNNE \cite{iNNE} employ the same partitioning mechanism. But the former is a distributional kernel and the latter is not.
iNNE employs a score which is a ratio of radii of two hyperspheres, designed to detect local anomalies \cite{iNNE}. The norm $\Vert \Phi(x) \Vert$ of Isolation Kernel, which is similar to the score of iNNE, simply counts the number of times $x$ falls outside of a set of all hyperspheres, out of $t$ sets of hyperspheres. This explains why both iNNE and  $\Vert \Phi(x) \Vert$ have similar AUCs for all the datasets
we used for point anomaly detection.

\begin{table}[t]
		\centering
		\caption{Results of isolation-based anomaly detectors (AUC). }
		\label{tbl_point_anomalies2}
		\begin{tabular}{lcccc}
			\toprule
			Dataset & $\widehat{\mathcal{K}}_I$ &$\Vert \Phi(x) \Vert$&iNNE&iForest  \\ \midrule
			speech &0.76&0.75&0.75&0.46\\
			EEG\_eye&0.88&0.87&0.87&0.58\\
			PenDigits&0.98&0.96&0.96&0.93\\
			MNIST\_230&0.98&0.97&0.97&0.88\\
			MNIST\_479&0.86&0.60&0.86&0.45\\
			mammograg &0.88&0.86&0.84&0.87\\
			electron  &0.80&0.78&0.79&0.80\\
			shuttle&0.98&0.98&0.98&0.99\\
			ALOI &0.82&0.82&0.82&0.55\\
			muon&0.82&0.81&0.82&0.74\\
			smtp&0.95&0.92&0.94&0.92\\
			IoT\_botnet&0.99&0.99&0.99&0.94\\
			ForestCover&0.97&0.96&0.96&0.93\\
			http&0.99&0.99&0.99&0.99\\
			\midrule
			{Average rank} &1.50 &2.75 & 2.43&3.32 
\\
			\bottomrule
		\end{tabular}
	\end{table}

In comparison with $\Vert \Phi(x) \Vert$ and iNNE, $\widehat{\mathcal{K}}_I$ has an additional distributional characterisation of the entire dataset.
Table~\ref{tbl_point_anomalies2} shows that $\widehat{\mathcal{K}}_I$'s score based on this characterization  leads to equal or better accuracy than the scores used by both $\Vert \Phi(x) \Vert$ and iNNE. This is because the characterisation provides an effective reference for point anomaly detection, robust to contamination of anomalies in a dataset. This robustness is important when using points in a dataset which contains anomalies to build a model (that consists of hyperspheres used in $\widehat{\mathcal{K}}_I$, $\Vert \Phi(x) \Vert$ and iNNE.)


In a nutshell, existing isolation-based anomaly detectors, i.e., iNNE and iForest, employ a score similar to the norm $\Vert \Phi(x) \Vert$. This is an interesting revelation because isolation-based anomaly detectors were never considered to be related to a kernel-based method before the current work. 
The power of isolation-based anomaly detectors can now be directly attributed to the norm of the feature map  $\Vert \Phi(x) \Vert$ of Isolation Kernel.

From another perspective, $\widehat{\mathcal{K}}_I$ can be regarded as a member of the family of Isolation-based anomaly detectors \cite{liu2008isolation, iNNE}; and iForest \cite{liu2008isolation} has been regarded as one of the state-of-the-art point anomaly detectors \cite{EmmottDDFW16,Aggarwal-Book2017}. iNNE \cite{iNNE}  is recently proposed to be an improvement of iForest.  \textbf{$\widehat{\mathcal{K}}_I$ is the only isolation-based anomaly detector that is also a kernel based anomaly detector.}


\section{Background on Group anomaly detection}
\label{sec_background_gad}
Group anomaly detection aims to detect group anomalies which exist in a dataset of groups of data points, where no labels are available, in an unsupervised learning context. 
Group anomaly detection requires a means to determine whether two groups of data points are generated from the same distribution or from different distributions.

There are two main approaches to the problem of group anomaly detection, depending on the learning methodology used. The generative approach aims to learn a model which represents the normal groups that can be used to differentiate from group anomalies.  Flexible Genre Model \cite{Xiong:2011} is an example of this approach. On the other hand, the discriminative approach aims to learn the smallest hypersphere which covers most of the training groups, assuming that the majority are normal groups. The representative is a kernel-based method called OCSMM \cite{OCSMM2013} (which is an extension of kernel point anomaly detector OCSVM \cite{OCSVM2001}).

Instead of focusing on learning methodologies, we introduce a third approach that focuses on using IDK to measure the degree of difference between two distributions. We show that IDK is a powerful kernel which can be used directly to differentiate group anomalies from normal groups, without explicit learning.

The intuition is that once every group in input space is mapped a point in Hilbert Space by IDK, then the point anomalies in Hilbert space (equivalent to group anomalies in input space) can be detected by using the point anomaly detector $\widehat{\mathcal{K}}_I$ we described in the last three sections. Here $\widehat{\mathcal{K}}_I$ is to be built from the dataset of points in Hilbert Space. This two-level application of IDK  is proposed to detect group anomalies in the next section.





\section{Proposed Method for\\ Group Anomaly Detection}
\label{sec_problem_setting}

The problem setting of group anomaly detection assumes a given dataset $M = \cup_{j=1}^n \mathsf{G}_j$, where $\mathsf{G}_j$ is a group of data points in $\mathbb{R}^d$; and $M= A \cup N$, 
where $A$ is the union of anomalous groups $\mathsf{G}_j$, and $N=M \backslash A$ is the union of normal groups. 
We denote an anomalous group $\mathsf{G} \subset A$ as $\mathsf{A}$; and a normal group $\mathsf{G} \subset N$ as $\mathsf{N}$.  
The membership of each $\mathsf{G}$, {\it i.e.}, $\mathsf{G}=\mathsf{A}$ or $\mathsf{N}$, is  not known; and it is to be discovered by a group anomaly detector.  

\begin{definition}
\textbf{
A group anomaly $\mathsf{A}$ is a set of data points, where every point in $\mathsf{A}$ is generated from distribution $\mathcal{P}_\mathsf{A}$. It has the following characteristics in a given dataset $M$:
\begin{itemize}
    \item [i.] $\mathcal{P}_\mathsf{A}$ is distinct from the distribution $\mathcal{P}_\mathsf{N}$ from which every point in any normal group $\mathsf{N}$ is generated such that the similarity between the two distributions is small, i.e., $\widehat{\mathcal{K}}(\mathcal{P}_\mathsf{N},\mathcal{P}_\mathsf{A}) \ll 1$;
    \item[ii.] The number of group anomalies is significantly smaller than the number of normal groups in $M$, i.e., $|A| \ll |N|$.
\end{itemize}
}
\label{def_group_anomalies}
\end{definition}
Note that the normal groups in a dataset $M$ may be generated from multiple distributions. This does not affect the definition, as long as each normal distribution $\mathcal{P}_\mathsf{N}$ satisfies the two above characteristics.

\subsection{Groups in input space are represented as points in Hilbert space associated with IDK}
Different from the two existing existing learning-based approaches mentioned in Section \ref{sec_background_gad}, we take a simpler approach that employs IDK, without learning. Given $M = \cup_{j=1}^n \mathsf{G}_j$,  it enables {\em each group $\mathsf{G}$ in the input space $\mathbb{R}^d$ to be represented  as a point $\widehat{\Phi}(\mathcal{P}_\mathsf{G} | M)$  in  the Hilbert space $\mathscr{H}$ associated with IDK $\widehat{\mathcal{K}}_I(\cdot | M)$ and its feature map $\widehat{\Phi}(\cdot | M)$}.

Let  $\Pi = \{ \widehat{\Phi}(\mathcal{P}_{\mathsf{G}_j}|M)\ |\ j=1,\dots,n \}$ denote the set of all IDK-mapped points in $\mathscr{H}$; and $\delta$ be a Dirac measure  defined in $\mathscr{H}$  which turns a point into a distribution. 

While our definition of group anomalies does not rely on a definition of point anomalies in input space, it does depend on a definition of point anomalies in Hilbert space.

\begin{definition}
\textbf{
Point anomalies $\mathbf{a} = \widehat{\Phi}(\mathcal{P}_\mathsf{A}|M)$ in $\mathscr{H}$ are points in $\Pi$  which have small similarities with respect to $\Pi$, i.e., $\widehat{\mathcal{K}}_I(\delta(\mathbf{a}),\mathbf{P}_\Pi\ |\ \Pi) \ll 1$.}
\end{definition}
Note that the similarity $\widehat{\mathcal{K}}_I$ is derived from $\Pi$ in this case in order to define point anomalies $\mathbf{a} \in \Pi$ in $\mathscr{H}$.

To simplify the notation, we use a shorthand (lhs) to denote the formal notation (rhs):
\[\widehat{\mathbf{K}}_I(\mathbf{a},\Pi) \equiv \widehat{\mathcal{K}}_I(
  \delta(\mathbf{a}),\mathbf{P}_\Pi | \Pi)\]

With the simplified notation, the similarity can be interpreted as one between a point and a set in Hilbert space.
 
\begin{definition}
\textbf{
Normal points  $\mathbf{n} = \widehat{\Phi}(\mathcal{P}_\mathsf{N}|M)$ in $\mathscr{H}$  are points in $\Pi$   which have large similarities with respect to $\Pi$ such that \[ \widehat{\mathbf{K}}_I(\mathbf{n},\Pi) > \widehat{\mathbf{K}}_I(\mathbf{a},\Pi) \]
}
\end{definition}

In the last two definitions, all points $\widehat{\Phi}(\cdot|M)$ in $\Pi$ are mapped using level-1 IDK derived from points in $M$ in the input space; and  $\widehat{\mathcal{K}}_I(\cdot,\cdot\ |\ \Pi)$ refers to the level-2 IDK derived from the level-1 IDK mapped points in $\Pi$. Recall that IDK is derived directly from a dataset, as described in Section \ref{sec_IDK}.





\subsection{IDK$^2$: No-learning IDK-based approach}
\label{sec_IDK2}

The above definitions of group anomalies lead directly to the proposed group anomaly detector called IDK$^2$
because of the use of two levels of IDK.

Though using the same kernel mean embedding, the proposed approach is a simpler but more effective alternative to that used in OCSMM \cite{OCSMM2013}.

Given the proposed distributional kernel IDK, 
the proposed group anomaly detector is simply applying IDK at two levels to a dataset consisting of groups $\mathsf{G}$: $M = \cup_{j=1}^n \mathsf{G}_j$:

\begin{enumerate}
  \item The level-1 IDK $\widehat{\Phi}(\cdot|M): \mathbb{R}^d \rightarrow \mathscr{H}$ maps  each  group of points ${\mathsf{G}_j}$ in the input space to a point  $\mathbf{g}_j = \widehat{\Phi}(\mathcal{P}_{\mathsf{G}_j}|M)$ in $\mathscr{H}$ to produce $\Pi$,  a set of points in $\mathscr{H}$, as follows: 
\[
\Pi = \{ 
\mathbf{g}_j\ |\ j=1,\dots,n\}\]

  \item The level-2 IDK $\widehat{\Phi}_2(\cdot|\Pi):  \mathscr{H} \rightarrow \mathscr{H}_2$  maps (i) $\mathbf{g}_j$ to $\widehat{\Phi}_2(\delta(\mathbf{g}_j)|\Pi)$; and (ii) $\Pi$ to $\widehat{\Phi}_2(\mathbf{P}_\Pi|\Pi)$. Then the IDK $\widehat{\mathcal{K}}_I(\delta(\mathbf{g}_j),\mathbf{P}_\Pi|\Pi)$  computes their similarity:
  \[\widehat{\mathbf{K}}_I(\mathbf{g}_j,\Pi) \equiv 
\frac{1}{t} \left<\Phi_2(\mathbf{g}_j),\widehat{\Phi}_2(\mathbf{P}_\Pi))\right>\]
where $\Phi_2(\mathbf{g}_j)=\widehat{\Phi}_2(\delta(\mathbf{g}_j))$.)

\end{enumerate}

\begin{figure*}[t]
  \vspace{-2mm}
\centering
\subfloat[Two levels of IDK mappings $\widehat{\Phi}$ \& $\widehat{\Phi}_2$ (and IK mapping ${\Phi}_2$ is associated with level-2 IDK)]{
\includegraphics[width=.7\textwidth]{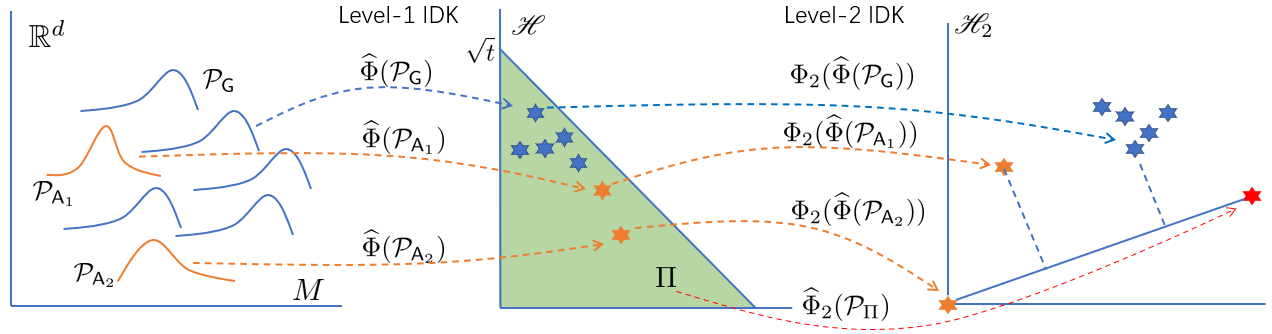} 
  \label{mapping_IDK}
}\\
\subfloat[Two levels of GDK mappings $\widehat{\varphi}$ \& $\widehat{\varphi}_2$ (and  GK mapping ${\varphi}_2$ is associated with level-2 GDK)]{
\includegraphics[width=.7\textwidth]{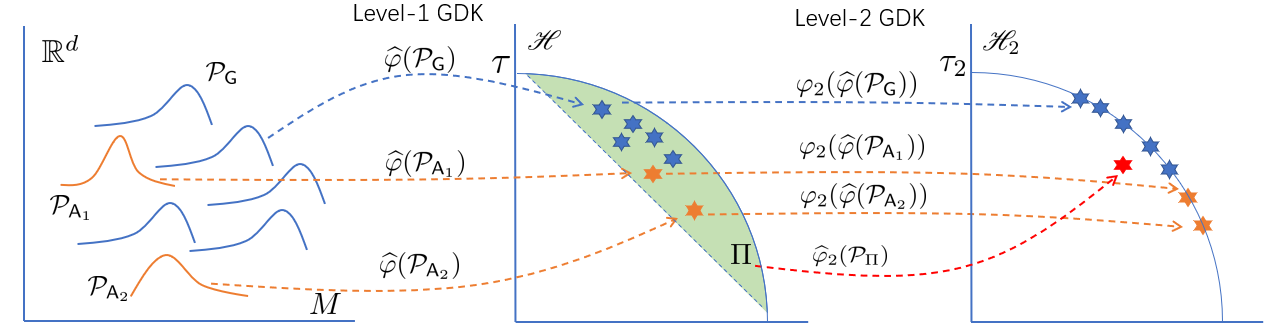} 
  \label{mapping_GDK}
}
\caption{The geometric interpretation of GDK and IDK mappings. (a) The projection of each point (e.g., ${\Phi}_2(\widehat{\Phi}(\mathcal{P}_G))$) in $\mathscr{H}_2$ onto the line from the origin to $\widehat{\Phi}_2(\mathcal{P}_\Pi)$ (in the right subfigure) is the similarity obtained from the dot product. This example shows that the level-2 mapping of $\mathsf{G}$ is more similar to that of $\Pi$ than that of $\mathsf{A}_1$ or $\mathsf{A}_2$. As a result, each normal group $\mathsf{G}$ is readily differentiated from group anomalies $\mathsf{A}_1$ or $\mathsf{A}_2$.  (b) This example shows that the level-2 GDK mapping  $\widehat{\varphi}_2(\mathcal{P}_\Pi)$ (in the right subfigure) is similar to the individual level-2 mappings of both the group anomalies $\mathsf{A}$ and some normal groups $\mathsf{G}$. Therefore, it is hard to differentiate the group anomalies from these normal groups.}
  \label{mapping}
  \vspace{-2mm}
\end{figure*}

The number of group anomalies is significantly smaller than the number of normal groups.
Thus, we assume that the distribution of the set of level-1 IDK-mapped points of normal groups $\mathbf{N}=\{ \mathbf{g}= \widehat{\Phi}(\mathcal{P}_{\mathsf{G}}|M) |\mathsf{G} \subset N\}$ is almost identical to the distribution of all level-1 IDK-mapped points $\Pi$, i.e., $\mathbf{P}_\mathbf{N} \approx \mathbf{P}_{\Pi}.$

Given this assumption, we use the similarity $\widehat{\mathbf{K}}_I(\mathbf{g},\Pi)$ to rank $\mathbf{g} \in \Pi$, i.e., {\em if the similarity is large,  then $\mathsf{G}$ is likely to be a normal group; otherwise it is an group anomaly}.

This procedure called IDK$^2$ is shown in Algorithm \ref{alg_IDK2}. Line~1 denotes the level-1 mapping $\widehat{\Phi}$; Line 3 denotes the level-2 mapping $\widehat{\Phi}_2$ (and its base mapping ${\Phi_2}$); and the similarity calculation of each $\mathsf{G}_j$ wrt $M$ after the two levels of mappings.

\begin{algorithm}[ht]
		\caption{IDK$^2$ algorithm }
		\label{alg_IDK2}
		\begin{algorithmic}[1]
 			\Require Dataset $M= \cup_{j=1}^n \mathsf{G}_j$; sample sizes $\psi$ of $M$ \& $\psi_2$ of $\Pi$ for two levels of IK mappings ${\Phi}$ \& ${\Phi}_2$, respectively.
 			\Ensure Group similarity score $\alpha_j$ for each group $\mathsf{G}_j \subset M$.
\State For each $j = 1,\dots,n,\  \mathbf{g}_j = \widehat{\Phi}(\mathcal{P}_{\mathsf{G}_j})$
\State $\Pi = \{ \mathbf{g}_j\ | \ j = 1,\dots,n\}$
\State For each $j = 1,\dots,n,\ \alpha_j = \frac{1}{t} \left<\Phi_2(\mathbf{g}_j),\widehat{\Phi}_2(\mathbf{P}_\Pi))\right>$
\State Sort $ \mathsf{G}_j \subset M$ in ascending order by $\alpha_j$
	\end{algorithmic}  
\end{algorithm}

Our use of $\widehat{\mathcal{K}}_I$ is a departure from the conventional use of (i) Distributional Kernel which measures similarity between two distributions; or (ii) point kernel which measures similarity between two points. We use $\widehat{\mathcal{K}}_I$ to measure between, effectively, a point and a set generated from an unknown distribution.
The use of a Distributional Kernel in this way is new, so as its interpretation, as far as we know.
Figure \ref{mapping_IDK} illustrates the two levels of IDK mapping from a distribution (generating a group) in the input space to a point in Hilbert space $\mathscr{H}$ where the set of all mapped points is $\Pi$. Level-2 mapping shows the similarity (as a dot product) of each point $\mathbf{g}_j \in \Pi$ wrt  $\Pi$ in the second level Hilbert space $\mathscr{H}_2$, i.e., $\widehat{\mathbf{K}}_I(\mathbf{g}_j,\Pi) =  \frac{1}{t} \left<\Phi_2(\mathbf{g}_j),\widehat{\Phi}_2(\mathbf{P}_\Pi))\right>$, as shown in Figure~\ref{mapping_IDK} in the right subfigure.

IDK$^2$ needs no explicit learning  to successfully rank group anomalies ahead of normal groups. Yet, the corresponding GDK$^2$ using the conventional GDK \cite{KernelMeanEmbedding2017} (that employs Gaussian kernel in Eq (\ref{eqn_mmk})) fails to achieve this. We explain the reason in the next section.

\section{Conceptual comparison and\\ Geometric interpretation}
\label{sec_conceptual_comparison}
Table \ref{tbl_IDK_vs_GDK} provides a summary comparison between Isolation Kernel (IK) and Gaussian kernel (GK), and their two levels of distributional kernels.

\setlength{\arrayrulewidth}{0.5mm}	
\begin{table}[t]
\vspace{-3mm}
		\centering
		\caption{\hspace{1cm}IK/IDK/IDK$^2$ 
		versus GK/GDK/GDK$^2$. \hspace{5cm} $x \sim \mathcal{P}$; $\mathbf{g} \sim \mathbf{P}_\Pi$; $\Pi$ is a set of points ($\mathbf{g} = \widehat{\Phi}(\mathcal{P})$ or $\widehat{\varphi}(\mathcal{P})$).
		}
		\label{tbl_IDK_vs_GDK}
		\vspace{-3mm}
		\begin{tabular}{@{}l@{}l@{\ }|@{\ }l@{}l@{}}
			\toprule
 IK:  &  $0 \leq \Vert {\Phi}(x) \Vert\ \leq \sqrt{t}$ & GK: & $\Vert \varphi(x) \Vert\ = \tau$ (constant)\\
 IDK: &  $0 \leq \Vert \widehat{\Phi}(\mathcal{P}) \Vert\ \leq \sqrt{t}$ & GDK: & $\Vert \widehat{\varphi}(\mathcal{P}) \Vert\ \leq \tau$\\
 IDK$^2$: & $0 \leq \widehat{\mathbf{K}}_I(\mathbf{g},\Pi) \leq \Vert \widehat{\Phi}_2(\mathbf{P}_\Pi) \Vert$ & GDK$^2$: & 
$\Vert \widehat{\varphi}_2(\mathbf{P}_\Pi) \Vert\ \leq \Vert {\varphi}_2(\mathbf{g}) \Vert = \tau_2$\\
	\bottomrule
		\end{tabular}
\vspace{-3mm}
\end{table}

\noindent
\textbf{GK vs IK}: Using Gaussian kernel (or any translation-invariant kernel, where its kernel depends on $x-x'$ only), $\kappa(x,x)$ is constant. Thus, $\Vert \varphi(x) \Vert = \tau$ (constant), and $\varphi(x)$ maps each point $x \in \mathbb{R}^d$ to a point on the sphere in $\mathscr{H}$ \cite{OCSMM2013}. $\Vert \varphi(x) \Vert = \tau$ also implies that GK is data independent. In contrast, IK is data dependent. $\kappa_I(x,x)$ is not constant; and its feature map has norm ranges:  $0 \leq \Vert {\Phi}(x) \Vert\ \leq \sqrt{t}$, where $t$ is a user setting of IK which could be interpreted as the number of effective dimensions of the feature map (see Section \ref{sec_IsolationKernel} for details.)

\noindent
\textbf{GDK \& OCSMM}:  GDK produces the kernel mean embedding $\widehat{\varphi}(\mathcal{P})$ that maps any group of points (generated from $\mathcal{P}$) into a region bounded by a plane (perpendicular to a vector from the origin of the sphere) and the sphere in $\mathscr{H}$ \cite{OCSMM2013}. 
    The desired distribution in $\mathscr{H}$ is: group anomalies are mapped by GDK away from the sphere, and normal groups are mapped close to the sphere. Then, a boundary plane can be used to exclude group anomalies from normal groups. See the geometric interpretation of level-1 mapping shown in Figure~\ref{mapping_GDK} (the middle subfigure), where the plane is shown as a line segment inside the sphere. 
    
    OCSMM, which replies on the GDK mapping, applies a learning process in $\mathscr{H}$ to determine the boundary  (equivalent to adjusting the plane mentioned above) based on a prior $\nu$ which denotes the proportion of anomalies in the training set. The boundary produces a `small' region (shaded in Figure~\ref{mapping_GDK} (the middle subfigure)) that captures most of the (normal) groups in the training set. A higher $\nu$ reduces the region so that more points in $\mathscr{H}$ (representing the group anomalies) in the training set can be excluded from the region.
    However, as we will see in the experiments later that there is no guarantee that GDK will map to the desired distribution in $\mathscr{H}$.
    
    \noindent
    \textbf{IDK \& IDK$^2$}: IDK has the distribution of points $\widehat{\Phi}(\mathcal{P})$ in $\mathscr{H}$ such that $0 \leq \Vert \widehat{\Phi}(\mathcal{P}) \Vert\ \leq \sqrt{t}$. See the geometric interpretation of level-1 mapping shown in Figure \ref{mapping_IDK} (the middle subfigure). Note that IDK is only required to map normal groups into neighbouring points in $\mathscr{H}$; and group anomalies are mapped to points away from the normal points in $\mathscr{H}$. The orientation of the distribution is immaterial, unlike GDK.  
    As long as level-1 IDK maps groups in input space to points having the above stated distribution in $\mathscr{H}$, the level-2 IDK will map group anomalies close to the origin of $\mathscr{H}_2$, and normal groups away from the origin of $\mathscr{H}_2$. Figure \ref{mapping_IDK} (the right subfigure) shows its geometric interpretation in $\mathscr{H}_2$. (This is because anomalies are likely to fall outside of the balls used to defined $\Phi_2$. See Section \ref{sec_IsolationKernel} for details.) As a result, no learning is required  in $\mathscr{H}_2$; and a simple similarity measurement using level-2 IDK will provide the required ranking to separate anomalies from normal points in $\mathscr{H}_2$. It computes  $\widehat{\mathbf{K}}_I(\mathbf{g},\Pi)$ between each level-1 mapped point $\mathbf{g}$ and $\Pi$ (the entire set of level-1 mapped points) via a dot product in $\mathscr{H}_2$. 
    
\noindent    
\textbf{GDK$^2$}: Figure \ref{mapping_GDK} (the right subfigure) shows the level-2 mapping of GDK. The individual groups are mapped on the sphere in $\mathscr{H}_2$; and $\Pi$ is mapped into a point inside the sphere. As a result, the similarity $\widehat{\mathbf{K}}(\mathbf{g},\Pi)$ could not be used to provide a proper ranking to separate anomaly groups from normal groups. This is apparent even when IDK is used at level 1 and GDK at level 2 (see the results of IDK-GDK in Table \ref{tbl_group_anomalies} in Section \ref{sec_experiments_gad}.)

\vspace{1mm}
The source of the power of IDK$^2$ comes from a special implementation of Isolation Kernel (IK) which maps anomalies close to the origin in Hilbert space. Also, its similarity is data dependent, i.e., two points
in a sparse region are more similar than two points
of equal inter-point distance in a dense region \cite{ting2018IsolationKernel,IsolationKernel-AAAI2019}. By the same token, GDK$^2$ is unable to use the similarity directly to identify group anomalies because it uses Gaussian kernel (GK) which is data independent.  

\section{Time complexities}
\label{sec_time_complexities}
We compare the time complexities of IDK, IDK$^2$, GDK, GDK$^2$ and OCSMM in this section. GDK$^2$ is our creation based on IDK$^2$, where the only difference is the use of Gaussian kernel instead of Isolation Kernel.
Their time complexities are summarised in Table \ref{tbl_time_complexity}.

The IDK preprocessing of $\widehat{\Phi}(\mathcal{P}_\mathsf{G})$ for $n$ groups $\mathsf{G}$ of $m$ points takes $O(nmt\psi)$ and needs to be completed once only. To compute the dot product of all pairs of level-1 mapped points in Hilbert space, IDK takes $O(n^2t\psi)$.

The dot product in the last step can be avoided when the level-2 IDK is used. This is because IDK$^2$ computes the similarity between each of the $n$ level-1 mapped point and the set of all level-1 mapped points in Hilbert space. This similarity is equivalent to the similarity between each of the $n$ groups and the set of all groups in input space.

\begin{table}[t]
	\centering
	\caption{Time complexities.}
	  \vspace{-1mm}
	\label{tbl_time_complexity}
	\begin{tabular}{ccccc}
	\toprule
	 IDK & IDK$^2$ & GDK & GDK$^2$ (Nystr\"{o}m) & OCSMM \\ 
\midrule
	 $nt\psi(m+n)$ & $nmt\psi$ & $n^2m^2$ & $nms^2$ & $n^2m^2$\\
	\bottomrule
	\end{tabular}
	  \vspace{-2mm}
\end{table}

		\begin{table*}[t]
		\centering
		\caption{A comparison of kernel based detectors for group anomalies (AUC).  \#Gr: \#Groups \& \#GA: \#Group anomalies.}
		\label{tbl_group_anomalies}  
		\begin{tabular}{lrrrr|cccccc}
			\toprule
			Dataset & \#Points &\#Gr & \#GA & \#Attr
			& IDK$^2$&GDK$^2$ & IDK-GDK & IDK-OCSVM &OCSMM&iNNE\\ \midrule
            SDSS$_{0\%}$ &7,530&505&10&500&0.99&0.56&0.99 & 0.73&0.97&0.63\\
			SDSS$_{50\%}$&7,530&505&10&500&0.99&0.55&0.98& 0.69&0.89&0.63\\	MNIST\_4&26,739&2,971&50&324&0.99&0.56&0.45& 0.60&0.78&0.54\\
			MNIST\_4$_{50\%}$&26,739&2,971&50&324&0.92&0.55&0.47& 0.59&0.72&0.50\\
			MNIST\_8&26,775&2,975&50&324&0.97&0.54&0.54& 0.71&0.81&0.55\\
			MNIST\_8$_{50\%}$&26,775&2,975&50&324&0.82&0.52&0.52& 0.60&0.75&0.57\\
			
			covtype&201,000 &2,010&10&54&0.73&0.74&0.42&0.58&0.62&0.48\\
			Speaker&119,246&4,505&5&20&0.69&0.51&0.59& 0.53&0.67&0.42\\
			Gaussian100&300,000&3,000&30&2&0.97&0.91&0.65&0.73&0.95&0.67\\
			Gaussian1K&1,040,000&1,040&40&2&1.00&0.76&0.38&0.79& $> 24$ hours&0.70\\
		
		   MixGaussian1K&1,040,000&1,040&40&2&0.99&0.67&0.73&0.74& $> 24$ hours&0.64\\
			
			\midrule
			\multicolumn{5}{r}{Average rank :}
			&1.14&4.27&4.50&3.27&3.00&4.82\\
			\bottomrule
		\end{tabular}
	\end{table*}

The level-2 kernel mean map of $n$ level-1 mapped points takes $O(nt\psi_2)$; the level-2 dot product for each of $n$ level-1 mapped points takes $O(nt\psi_2)$. Thus, the total time complexity for two levels of mappings and the level-2 dot product is $O(nmt\psi)$ (assume $\psi_2 = \psi$.)

GDK takes $O(n^2m^2)$.
However, like IDK$^2$, GDK$^2$ can also make use of a finite-dimensional feature map to speed its runtime.
The Nystr\"{o}m  method \cite{Nystrom_NIPS2000} can be used to produce an approximate finite-dimensional feature map of Gaussian kernel. It uses a subset of ``landmark'' data points instead of all the data points. Given $s$ landmark points, the kernel matrix in factored form takes $O(nms)$ kernel evaluations and $O(s^3)$ additional time. 
The landmark points, selected by the recursive sampling algorithm~\cite{musco2017recursive}, takes an extra $O(nms^2)$. 
The total time complexity for GDK$^2$, accelerated by the Nystr\"{o}m  method,  is $O(nms^2+s^3)$. Since $s \ll nm$, the complexity is thus $O(nms^2)$ (assume $s = s_2$ for two levels of mappings.)

This shows that \emph{the use of a distributional kernel, which has a finite-dimensional feature map to compute the similarity between (effectively) a point and a dataset in Hilbert space, is the key in producing a linear time kernel based algorithm} in both IDK$^2$ and GDK$^2$. Otherwise, they could cost $O(n^2m^2)$.

\section{Experiments: Group Anomaly Detection}
\label{sec_experiments_gad}
We compare 
 IDK$^2$, GDK$^2$ and IDK-GDK  with 
 a state-of-the-art OCSMM \cite{OCSMM2013} as well as  IDK-OCSVM. The prefix `IDK-' indicates that the level-1 IDK mapped points are produced as input to GDK or OCSVM which is used as a representative kernel based point anomaly detector.
OCSMM employs kernel mean embedding of Eq \ref{eqn_mmk} where $\kappa$ is Gaussian kernel; and IDK-OCSVM employs linear kernel---it is equivalent to OCSMM using IDK. A recent point anomaly detector iNNE \cite{iNNE} is included as a strawman to demonstrate that the group anomalies in each dataset cannot be  identified effectively by using a point anomaly detector. iNNE is chosen because isolation-based methods have been shown to be a state-of-the-art point anomaly detector \cite{EmmottDDFW16,Aggarwal-Book2017,iNNE}.


The experimental settings\footnote{The sources of the datasets used are:\\sdss3.org, odds.cs.stonybrook.edu, sites.google.com/site/gspangsite,\\csie.ntu.edu.tw/$\sim$cjlin/libsvmtools/datasets/,\\ lapad-web.icmc.usp.br/repositories/outlier-evaluation/{DAMI}.} are the same as stated in Section \ref{sec_experiments}.

\subsection{Detection accuracy}

Table \ref{tbl_group_anomalies} shows that IDK$^2$ is the most effective group anomaly detector. The strawman iNNE point anomaly detector performed the worst.
The closest contender is OCSMM; but it performed worse than IDK$^2$ in all datasets. As shown in Figure~\ref{friedman_GAD}, the Friedman-Nemenyi test \cite{NemenyiTest-2006} 
shows that IDK$^2$  is significantly better than OCSMM at significance level 0.05.
OCSMM can be viewed as an improvement over GDK that needs an additional learning process, as illustrated in Figure~\ref{mapping_GDK}.
In contrast, IDK$^2$, which employs a data dependent kernel, does not need learning; yet it performed significantly better than OCSMM.

\begin{figure}[t]
\centering
    \includegraphics[height=.08\textwidth,width=.28\textwidth]{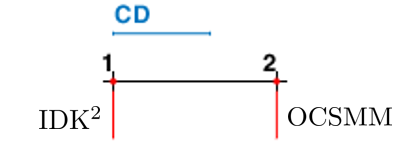}
  \caption{Friedman-Nemenyi test for group anomaly detectors. }
\label{friedman_GAD}
\end{figure}

GDK$^2$ is equivalent to IDK$^2$, but it employs a data independent kernel. The huge difference in detection accuracies between these two methods is startling---showing the importance of data dependency. 
The result of IDK-OCSVM (i.e., OCSMM using IDK) shows that the learning procedure used in OCSVM or OCSMM could not be used for IDK because of the differences in mappings shown in Figure \ref{mapping}. The procedure seeks to optimize the hyperplane excluding a few groups in $\mathscr{H}$. This does not work with the level-1 IDK mapping shown in Figure \ref{mapping_IDK}. 
Note that OCSMM could not complete each of the two datasets having one million data points within 24 hours. Yet, IDK$^2$ completed them in 1 hour.

\subsection{IDK$^2$ versus OCSMM}
\label{sec_compare_with_OCSMM}
Here we perform a more thorough investigation in determining the condition under which one is better than the other between IDK$^2$ and its closest contender OCSMM. We conduct two experiments using artificial datasets.

\noindent
\textbf{First experiment}:
We use an one-dimensional dataset of a total of 1500 groups, where each group of 100 points is generated from a Gaussian distribution of a different $\mu$ in $[-15,25]$. The groups are clustered into three density peaks with the highest at $\mu=18$, the second highest at $\mu=8$ and the third highest at $\mu=-2$. By using the center of each Gaussian distribution as the representative point for each group, we estimate the density of the dataset of 1500 representative points using KDE: Kernel Density Estimator (with Gaussian kernel).

We train IDK$^2$ and OCSMM with this dataset; and report their similarity/score for each group. 
Figure \ref{fig_local_anomalies_test} shows the result of the comparison. The distribution of IDK$^2$ mirrors the three peaks of the density distribution of the dataset. If we regard hard-to-detect group anomalies as points having similarity between 0.25 and 0.75 in Figure \ref{fig_local_anomalies_test}, then IDK$^2$ identifies all these group anomalies surrounding all three clusters. Easy-to-detect group anomalies are those points having similarities $< 0.25$ at the fringes; and normal groups are those $> 0.75$ surrounding each of the three peaks. Note that the cut-offs of 0.25 and 0.75 are arbitrary for the purpose of demonstration only.

\begin{figure}[t]
\centering
\includegraphics[width=.33\textwidth]{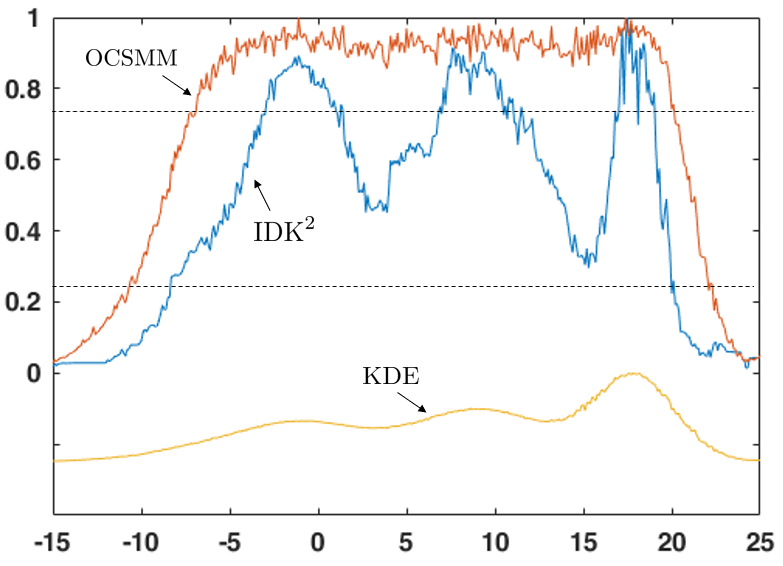}
\caption{Distributions of similarities of IDK$^2$ and OCSMM on an one-dimensional dataset which is generated by three Gaussian distributions. The y-axis shows the similarity of IDK$^2$ or the score of OCSMM. The dataset's density distribution is estimated using KDE (scale not shown). }
\label{fig_local_anomalies_test}
\end{figure}

In contrast, OCSMM fails to detect many hard-to-detect group anomalies, especially those in-between the three peaks, because none of the three peaks can be identified from the output distribution of OCSMM. It can only identify easy-to-detect group anomalies located at the fringes.  

\begin{figure}[!t]
\vspace{-2mm}
\subfloat[IDK$^2$]{
\label{subfig_IDK}
\includegraphics[width=.235\textwidth]{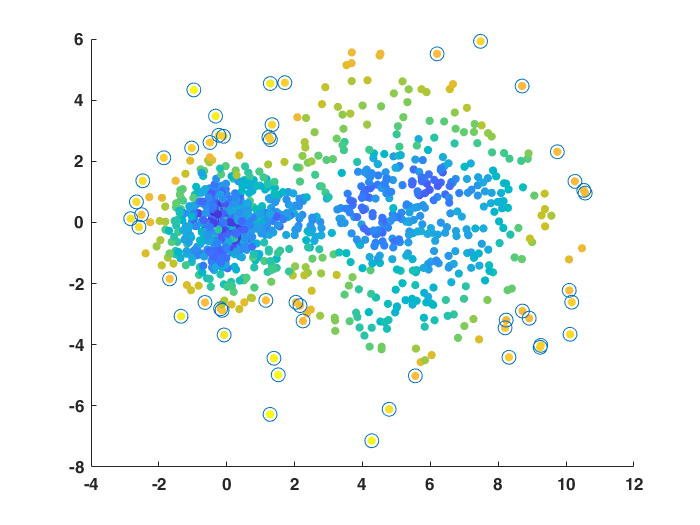}}
\subfloat[OCSMM]{
\label{subfig_smm}
\includegraphics[width=.235\textwidth]{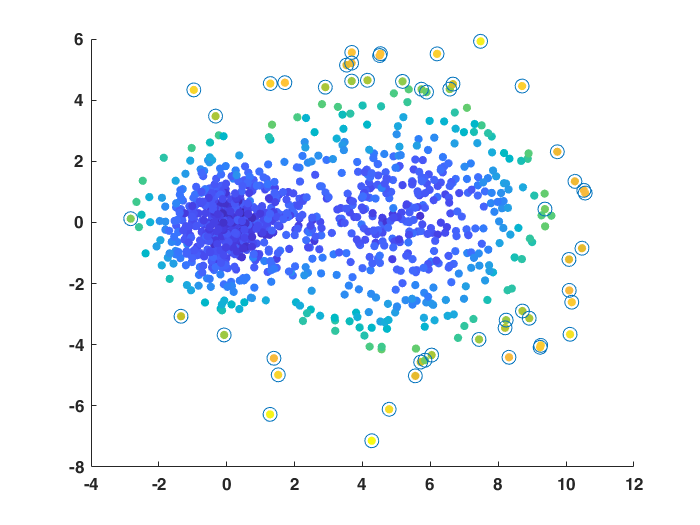}}
\caption{
Top 50 group anomalies (indicated as circles) as identified by IDK$^2$ and OCSMM on a dataset, where each point in the plot is a representative point of a group of 100 data points generated from a Gaussian distribution, each has a different mean but having the same variance.
The color indicates the level of similarity/score.}
\label{fig_two-clusters}
\end{figure}

\noindent
\textbf{Second experiment}: We use a two-dimensional dataset having two clusters of groups of points as shown in Figure \ref{fig_two-clusters}, where one cluster is denser than the other. 

The result in Figure \ref{fig_two-clusters}  shows that IDK$^2$ has a balance of anomalies identified from both clusters in the top 50 group anomalies. In contrast, the majority of anomalies identified by OCSMM came from the sparse cluster, and it missed out many anomalies associated with the dense cluster.



In a nutshell, both experiments show that (i) OCSMM fails to identify hard-to-detect group anomalies  in scenarios where there are clusters of normal groups with varied densities; and (ii) IDK$^2$ can detect both hard-to-detect and easy-to-detect group anomalies more effectively.

\begin{figure}[!t]
\centering
\subfloat[Scaleup test]{\includegraphics[width=.23\textwidth]{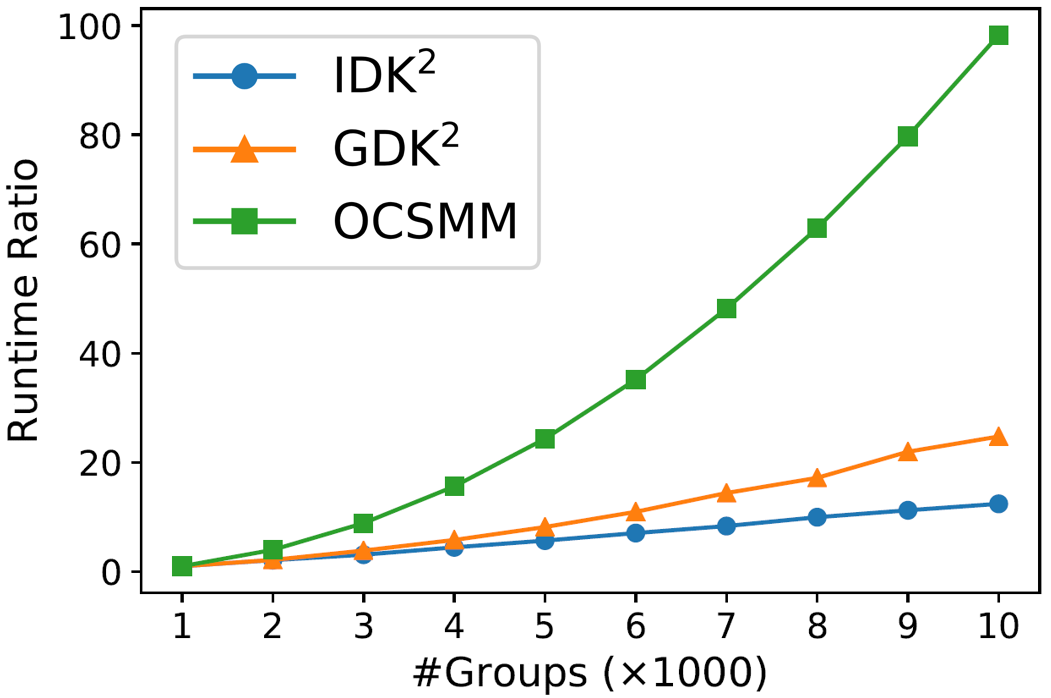}
\label{fig:timecost}}
\subfloat[IDK$^2$: CPU vs GPU]{ \includegraphics[width=.23\textwidth]{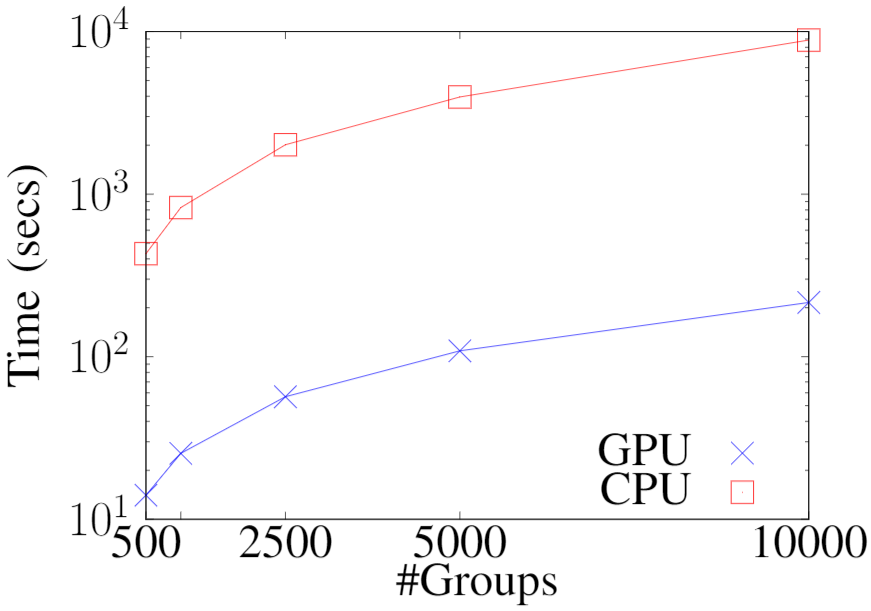} \label{fig:cpu-vs-gpu}}
\caption{Runtime comparison}

\label{fig:timecost_comparison}
\end{figure}

\subsection{Scaleup test}
\label{sec_scaleup_test}
We perform a scaleup test for IDK$^2$, GDK$^2$ and OCSMM by increasing the number of groups from 1,000 to 10,000 on an artificial two-dimensional dataset. 
Figure~\ref{fig:timecost} shows that OCSMM has the largest runtime increase; and IDK$^2$ has the lowest. When the number of groups increases 10 times from 1,000 to 10,000, OCSMM has its runtime increased by a factor of 98; while IDK$^2$ and GDK$^2$ have their runtimes increased by a factor of 12 and 25, respectively---linear to the number of groups.
IDK$^2$ using the IK implementation (stated in Section \ref{sec_IsolationKernel}) is amenable to GPU acceleration.
Figure \ref{fig:cpu-vs-gpu} shows the runtime comparison between the CPU and GPU implementations of IDK$^2$.  
At 10,000 groups, the CPU version took 8,864 seconds; whereas the GPU version took 216 seconds.

\section{Conclusions}
We show that the proposed Isolation Distributional Kernel (IDK) addresses two key issues of  kernel mean embedding, where the point kernel employed has: (i) a feature map with intractable dimensionality which leads to high computational cost; and (ii) data independency which leads to poor detection accuracy in anomaly detection.

We introduce a new Isolation Kernel and establish the geometrical interpretation of its feature map. We provide proofs of its data dependent property and a characteristic kernel---an essential criterion when a point kernel is used in the kernel mean embedding framework. 

The geometrical interpretation provides the insight that point anomalies and normal points are mapped into distinct regions in the feature space. This is the source of the power of IDK.
We also reveal that the distributional characterisation makes IDK robust to contamination of point anomalies.

Our evaluation shows that the proposed anomaly detectors IDK and IDK$^2$, without explicit learning,  produce better detection accuracy than existing key kernel-based methods in detecting both point and group anomalies. They also run orders of magnitude faster in large datasets.

In group anomaly detection, the superior runtime performance of IDK$^2$ is due to two main factors. First, the use of a distributional kernel to compute similarity between (effectively) a group and a dataset of groups.
This has avoided the need to compute all pairs of groups which is the root cause of the high computational cost of OCSMM. 
Second,  our proposed implementation of Isolation Kernel used in IDK$^2$ also enables GPU acceleration. We show that the GPU version runs more than one order of magnitude faster than the CPU version, where the latter is already orders of magnitude faster than OCSMM.

This work reveals  for the first time that the problem of detecting group anomalies in input space can be effectively reformulated as the problem of point anomaly detection in Hilbert space, without explicit learning.


%



\ifCLASSOPTIONcompsoc
  \section*{Acknowledgments}
\else
  \section*{Acknowledgment}
\fi

This paper is an extended version of our earlier work on using IDK for point anomaly detection only, reported in KDD-2020~\cite{ting2020isolation}.

\ifCLASSOPTIONcaptionsoff
  \newpage
\fi



%
\bibliographystyle{IEEEtran}
\bibliography{references}

\begin{thebibliography}{10}
\providecommand{\url}[1]{#1}
\csname url@samestyle\endcsname
\providecommand{\newblock}{\relax}
\providecommand{\bibinfo}[2]{#2}
\providecommand{\BIBentrySTDinterwordspacing}{\spaceskip=0pt\relax}
\providecommand{\BIBentryALTinterwordstretchfactor}{4}
\providecommand{\BIBentryALTinterwordspacing}{\spaceskip=\fontdimen2\font plus
\BIBentryALTinterwordstretchfactor\fontdimen3\font minus
  \fontdimen4\font\relax}
\providecommand{\BIBforeignlanguage}[2]{{%
\expandafter\ifx\csname l@#1\endcsname\relax
\typeout{** WARNING: IEEEtran.bst: No hyphenation pattern has been}%
\typeout{** loaded for the language `#1'. Using the pattern for}%
\typeout{** the default language instead.}%
\else
\language=\csname l@#1\endcsname
\fi
#2}}
\providecommand{\BIBdecl}{\relax}
\BIBdecl

\bibitem{EMK_Bo-NIPS2009}
L.~Bo and C.~Sminchisescu, ``Efficient match kernels between sets of features
  for visual recognition,'' in \emph{Proceedings of International Conf. on
  Neural Information Processing Systems}, 2009, pp. 135--143.

\bibitem{SupportMeasureMchines-NIPS2012}
K.~Muandet, K.~Fukumizu, F.~Dinuzzo, and B.~Scholkopf, ``Learning from
  distributions via support measure machines,'' in \emph{Advances in Neural
  Information Processing Systems}, 2012, pp. 10--18.

\bibitem{Sutherland-Thesis2016}
D.~J. Sutherland, \emph{Scalable, Flexible and Active Learning on
  Distributions}.\hskip 1em plus 0.5em minus 0.4em\relax PhD Thesis, School of
  Computer Science, Carnegie Mellon University, 2016.

\bibitem{HilbertSpaceEmbedding2007}
A.~Smola, A.~Gretton, L.~Song, and B.~Sch{\"o}lkopf, ``A hilbert space
  embedding for distributions,'' in \emph{Algorithmic Learning Theory},
  M.~Hutter, R.~A. Servedio, and E.~Takimoto, Eds.\hskip 1em plus 0.5em minus
  0.4em\relax Springer, 2007, pp. 13--31.

\bibitem{KernelMeanEmbedding2017}
K.~Muandet, K.~Fukumizu, B.~Sriperumbudur, and B.~Schölkopf, ``Kernel mean
  embedding of distributions: A review and beyond,'' \emph{Foundations and
  Trends in Machine Learning}, vol. 10 (1–2), pp. 1--141, 2017.

\bibitem{ting2018IsolationKernel}
K.~M. Ting, Y.~Zhu, and Z.-H. Zhou, ``Isolation kernel and its effect on
  {SVM},'' in \emph{Proceedings of the ACM SIGKDD International Conference on
  Knowledge Discovery and Data Mining}, 2018, pp. 2329--2337.

\bibitem{IsolationKernel-AAAI2019}
X.~Qin, K.~M. Ting, Y.~Zhu, and V.~C.~S. Lee, ``Nearest-neighbour-induced
  isolation similarity and its impact on density-based clustering,'' in
  \emph{Proceedings of The Thirty-Third AAAI Conference on Artificial
  Intelligence}, 2019, pp. 4755--4762.

\bibitem{Nystrom_NIPS2000}
C.~K.~I. Williams and M.~Seeger, ``Using the {N}ystr\"{o}m method to speed up
  kernel machines,'' in \emph{Advances in Neural Information Processing Systems
  13}, 2001, pp. 682--688.

\bibitem{RandomFeatures2007}
A.~Rahimi and B.~Recht, ``Random features for large-scale kernel machines,'' in
  \emph{Proceedings of the 20th International Conference on Neural Information
  Processing Systems}, 2007, pp. 1177--1184.

\bibitem{Nystrom-NIPS12}
T.~Yang, Y.-F. Li, M.~Mahdavi, R.~Jin, and Z.-H. Zhou, ``Nystr\"{o}m method vs
  random fourier features: A theoretical and empirical comparison,'' in
  \emph{Proceedings of the 25th International Conference on Neural Information
  Processing Systems - Volume 1}, ser. NIPS'12.\hskip 1em plus 0.5em minus
  0.4em\relax USA: Curran Associates Inc., 2012, pp. 476--484.

\bibitem{DistMetricLearning-Xing:2002}
E.~P. Xing, A.~Y. Ng, M.~I. Jordan, and S.~Russell, ``Distance metric learning,
  with application to clustering with side-information,'' in \emph{Proceedings
  of the 15th International Conference on Neural Information Processing
  Systems}, 2002, pp. 521--528.

\bibitem{weinberger2009distance}
K.~Q. Weinberger and L.~K. Saul, ``Distance metric learning for large margin
  nearest neighbor classification,'' \emph{Journal of Machine Learning
  Research}, vol.~10, no.~2, pp. 207--244, 2009.

\bibitem{zadeh2016geometric}
P.~Zadeh, R.~Hosseini, and S.~Sra, ``Geometric mean metric learning,'' in
  \emph{International Conference on Machine Learning}, 2016, pp. 2464--2471.

\bibitem{iNNE}
T.~R. Bandaragoda, K.~M. Ting, D.~Albrecht, F.~T. Liu, Y.~Zhu, and J.~R. Wells,
  ``Isolation-based anomaly detection using nearest neighbour ensembles,''
  \emph{Computational Intelligence}, vol.~34, no.~4, pp. 968--998, 2018.

\bibitem{Fukunaga-Book1990}
F.~Keinosuke, \emph{Introduction to Statistical Pattern Recognition}.\hskip 1em
  plus 0.5em minus 0.4em\relax Academic Press, 1990, ch.~6, pp. 268--270.

\bibitem{SoftMatter2014}
V.~Baranau and U.~Tallarek, ``Random-close packing limits for monodisperse and
  polydisperse hard spheres,'' \emph{Soft Matter}, vol.~10, pp. 3826--3841,
  2014.

\bibitem{Nature2008}
C.~Song, P.~Wang, and H.~A. Makse, ``A phase diagram for jammed matter,''
  \emph{Nature}, vol. 453, no. 7195, pp. 629--632, 2008.

\bibitem{Sriperumbudur:2010}
B.~K. Sriperumbudur, A.~Gretton, K.~Fukumizu, B.~Sch\"{o}lkopf, and G.~R.
  Lanckriet, ``Hilbert space embeddings and metrics on probability measures,''
  \emph{Journal of Machine Learning Research}, vol.~11, pp. 1517--1561, 2010.

\bibitem{musco2017recursive}
C.~Musco and C.~Musco, ``Recursive sampling for the nystrom method,'' in
  \emph{Adv in Neural Information Processing Systems}, 2017, pp. 3833--3845.

\bibitem{OCSVM2001}
B.~Sch\"{o}lkopf, J.~C. Platt, J.~C. Shawe-Taylor, A.~J. Smola, and R.~C.
  Williamson, ``Estimating the support of a high-dimensional distribution,''
  \emph{Neural Computing}, vol.~13, no.~7, pp. 1443--1471, 2001.

\bibitem{KDEOS}
E.~Schubert, A.~Zimek, and H.-P. Kriegel, ``Generalized outlier detection with
  flexible kernel density estimates,'' in \emph{Proceedings of the SIAM
  Conference on Data Mining}, 2014, pp. 542--550.

\bibitem{Local_kernel_density}
W.~{Hu}, J.~{Gao}, B.~{Li}, O.~{Wu}, J.~{Du}, and S.~{Maybank}, ``Anomaly
  detection using local kernel density estimation and context-based
  regression,'' \emph{IEEE Transactions on Knowledge and Data Engineering},
  vol.~32, no.~2, pp. 218--233, 2020.

\bibitem{FastLOF-TKDE2016}
M.~{Salehi}, C.~{Leckie}, J.~C. {Bezdek}, T.~{Vaithianathan}, and X.~{Zhang},
  ``Fast memory efficient local outlier detection in data streams,'' \emph{IEEE
  Transactions on Knowledge and Data Engineering}, vol.~28, no.~12, pp.
  3246--3260, 2016.

\bibitem{NemenyiTest-2006}
J.~Dem\v{s}ar, ``Statistical comparisons of classifiers over multiple data
  sets,'' \emph{Journal of Machine Learning Research}, pp. 1--30, 2006.

\bibitem{liu2008isolation}
F.~T. Liu, K.~M. Ting, and Z.-H. Zhou, ``Isolation forest,'' in
  \emph{Proceedings of IEEE International Conference on Data Mining}, 2008, pp.
  413--422.

\bibitem{EmmottDDFW16}
A.~Emmott, S.~Das, T.~G. Dietterich, A.~Fern, and W.~Wong, ``A meta-analysis of
  the anomaly detection problem,'' \emph{CoRR}, 2016.

\bibitem{Aggarwal-Book2017}
C.~C. Aggarwal and S.~Sathe, \emph{Outlier Ensembles: An Introduction}.\hskip
  1em plus 0.5em minus 0.4em\relax Springer International Publishing, 2017.

\bibitem{Xiong:2011}
L.~Xiong, B.~P\'{o}czos, and J.~Schneider, ``Group anomaly detection using
  flexible genre models,'' in \emph{Proceedings of the International Conference
  on Neural Information Processing Systems}, 2011, pp. 1071--1079.

\bibitem{OCSMM2013}
K.~Muandet and B.~Sch\"{o}lkopf, ``One-class support measure machines for group
  anomaly detection,'' in \emph{Proceedings of the Twenty-Ninth Conference on
  Uncertainty in Artificial Intelligence}, 2013, pp. 449--458.

\bibitem{ting2020isolation}
K.~M. Ting, B.-C. Xu, T.~Washio, and Z.-H. Zhou, ``Isolation distributional
  kernel: A new tool for kernel based anomaly detection,'' in \emph{Proceedings
  of the 26th ACM SIGKDD International Conference on Knowledge Discovery \&
  Data Mining}, 2020, pp. 198--206.

\end{thebibliography}

\end{document}